\documentclass[twoside]{article}

\usepackage[accepted]{aistats_like_template}

\usepackage{natbib}

\usepackage{etoolbox}

\makeatletter
\newcommand*{\sortentry}[1]{%
  \if@filesw
    \immediate\write\@auxout{\string\scNAT@aux@sortentry{#1}}%
  \fi}
\newcommand*{\scNAT@aux@sortentry}{%
  \listgadd{\scNAT@bibsortlist}}
\newcommand*{\scNAT@bibsortlist}{}

\newcommand*{\scNAT@citekeys}{}

\newcommand*{\scNAT@writetocitelistsort}[1]{%
  \ifinlist{#1}{\scNAT@citekeys}
    {\ifdefvoid{\NAT@cite@list}
       {\def\NAT@cite@list{#1}}
       {\expandafter\def\expandafter\NAT@cite@list\expandafter{\NAT@cite@list,#1}}%
     \listgadd{\scNAT@foundkeys}{#1}}
    {}}

\newcommand*{\scNAT@writetocitelistforgotten}[1]{%
  \ifinlist{#1}{\scNAT@foundkeys}
    {}
    {\ifdefvoid{\NAT@cite@list}
       {\def\NAT@cite@list{#1}}
       {\expandafter\def\expandafter\NAT@cite@list\expandafter{\NAT@cite@list,#1}}}}

\newcommand*{\scNAT@sortcites}[1]{%
  \let\NAT@cite@list\@empty
  \let\scNAT@citekeys\@empty
  \let\scNAT@foundkeys\@empty
  \forcsvlist{\listadd{\scNAT@citekeys}}{#1}%
  \forlistloop{\scNAT@writetocitelistsort}{\scNAT@bibsortlist}%
  \forlistloop{\scNAT@writetocitelistforgotten}{\scNAT@citekeys}%
}

\def\NAT@citex%
  [#1][#2]#3{%
  \NAT@reset@parser
  \NAT@sort@cites{#3}%
  \scNAT@sortcites{#3}%<- this is new
  \NAT@reset@citea
  \@cite{\let\NAT@nm\@empty\let\NAT@year\@empty
    \@for\@citeb:=\NAT@cite@list\do
    {\@safe@activestrue
     \edef\@citeb{\expandafter\@firstofone\@citeb\@empty}%
     \@safe@activesfalse
     \@ifundefined{b@\@citeb\@extra@b@citeb}{\@citea%
       {\reset@font\bfseries ?}\NAT@citeundefined
                 \PackageWarning{natbib}%
       {Citation `\@citeb' on page \thepage \space undefined}\def\NAT@date{}}%
     {\let\NAT@last@nm=\NAT@nm\let\NAT@last@yr=\NAT@year
      \NAT@parse{\@citeb}%
      \ifNAT@longnames\@ifundefined{bv@\@citeb\@extra@b@citeb}{%
        \let\NAT@name=\NAT@all@names
        \global\@namedef{bv@\@citeb\@extra@b@citeb}{}}{}%
      \fi
     \ifNAT@full\let\NAT@nm\NAT@all@names\else
       \let\NAT@nm\NAT@name\fi
     \ifNAT@swa\ifcase\NAT@ctype
       \if\relax\NAT@date\relax
         \@citea\NAT@hyper@{\NAT@nmfmt{\NAT@nm}\NAT@date}%
       \else
         \ifx\NAT@last@nm\NAT@nm\NAT@yrsep
            \ifx\NAT@last@yr\NAT@year
              \def\NAT@temp{{?}}%
              \ifx\NAT@temp\NAT@exlab\PackageWarningNoLine{natbib}%
               {Multiple citation on page \thepage: same authors and
               year\MessageBreak without distinguishing extra
               letter,\MessageBreak appears as question mark}\fi
              \NAT@hyper@{\NAT@exlab}%
            \else\unskip\NAT@spacechar
              \NAT@hyper@{\NAT@date}%
            \fi
         \else
           \@citea\NAT@hyper@{%
             \NAT@nmfmt{\NAT@nm}%
             \hyper@natlinkbreak{%
               \NAT@aysep\NAT@spacechar}{\@citeb\@extra@b@citeb
             }%
             \NAT@date
           }%
         \fi
       \fi
     \or\@citea\NAT@hyper@{\NAT@nmfmt{\NAT@nm}}%
     \or\@citea\NAT@hyper@{\NAT@date}%
     \or\@citea\NAT@hyper@{\NAT@alias}%
     \fi \NAT@def@citea
     \else
       \ifcase\NAT@ctype
        \if\relax\NAT@date\relax
          \@citea\NAT@hyper@{\NAT@nmfmt{\NAT@nm}}%
        \else
         \ifx\NAT@last@nm\NAT@nm\NAT@yrsep
            \ifx\NAT@last@yr\NAT@year
              \def\NAT@temp{{?}}%
              \ifx\NAT@temp\NAT@exlab\PackageWarningNoLine{natbib}%
               {Multiple citation on page \thepage: same authors and
               year\MessageBreak without distinguishing extra
               letter,\MessageBreak appears as question mark}\fi
              \NAT@hyper@{\NAT@exlab}%
            \else
              \unskip\NAT@spacechar
              \NAT@hyper@{\NAT@date}%
            \fi
         \else
           \@citea\NAT@hyper@{%
             \NAT@nmfmt{\NAT@nm}%
             \hyper@natlinkbreak{\NAT@spacechar\NAT@@open\if*#1*\else#1\NAT@spacechar\fi}%
               {\@citeb\@extra@b@citeb}%
             \NAT@date
           }%
         \fi
        \fi
       \or\@citea\NAT@hyper@{\NAT@nmfmt{\NAT@nm}}%
       \or\@citea\NAT@hyper@{\NAT@date}%
       \or\@citea\NAT@hyper@{\NAT@alias}%
       \fi
       \if\relax\NAT@date\relax
         \NAT@def@citea
       \else
         \NAT@def@citea@close
       \fi
     \fi
     }}\ifNAT@swa\else\if*#2*\else\NAT@cmt#2\fi
     \if\relax\NAT@date\relax\else\NAT@@close\fi\fi}{#1}{#2}}
\makeatother

\usepackage[utf8]{inputenc}

\usepackage{amsfonts}
\usepackage{amsmath}
\usepackage{amssymb}
\usepackage{amsthm}

\usepackage{numprint}
\usepackage{tcolorbox}
\usepackage{subfig}
\usepackage{enumitem}
\usepackage{multirow}
\usepackage{url}
\usepackage{booktabs}       % professional-quality tables
\usepackage{adjustbox}
\usepackage{xspace}
\usepackage{physics}
\usepackage{xcolor}
\usepackage{hyperref}
\hypersetup{
    colorlinks,
    linkcolor={blue!90!black},
    citecolor={blue!90!black},
    urlcolor={blue!90!black}
}

\usepackage{tikz}
\usetikzlibrary{bayesnet}

\newtheorem{definition}{Definition}[section]

\newtheorem{lemma}{Lemma}[section]

\newtheorem{corollary}{Corollary}[section]

\newcommand{\para}[1]{\medskip \noindent {\bf #1}}

\newcommand{\bz}{\boldsymbol z}
\newcommand{\bZ}{\boldsymbol Z}
\newcommand{\bS}{\boldsymbol S}
\newcommand{\bA}{\boldsymbol A}
\newcommand{\bh}{\boldsymbol h}

\newcommand{\be}{\boldsymbol e}

\newcommand{\bY}{\boldsymbol Y}
\newcommand{\bx}{\boldsymbol x}

\newcommand{\bX}{\boldsymbol X}

\newcommand{\bD}{\boldsymbol D}
\newcommand{\bs}{\boldsymbol s}

\newcommand{\bzero}{\boldsymbol 0}

\newcommand{\bt}{\boldsymbol t}

\newcommand{\boldEta}{\boldsymbol H}
\newcommand{\bI}{\boldsymbol I}
\newcommand{\btheta}{\boldsymbol \theta}

\newcommand{\bOmega}{\boldsymbol \Omega}
\newcommand{\bmu}{\boldsymbol \mu}
\newcommand{\bzeta}{\boldsymbol \zeta}
\newcommand{\bSigma}{\boldsymbol \Sigma}

\begin{document}

\twocolumn[
\aistatstitle{Locally Differentially Private Bayesian Inference}
\aistatsauthor{Tejas Kulkarni$^{*,1}$ \And Joonas Jälkö$^{*,1}$ \And  Samuel Kaski$^{1,2}$ \And Antti Honkela$^{3}$ }
\aistatsaddress{Aalto University$^{1}$ \And  University of Manchester$^{2}$ \And University of Helsinki$^{3}$ }
]

%\graphicspath{{../figures/}}

\begin{abstract}
In recent years, local differential privacy (LDP) has  emerged as a technique of choice for privacy-preserving data collection in several scenarios when the aggregator is  not trustworthy. LDP provides client-side privacy by adding noise at the user's end. Thus, clients need not rely on the trustworthiness of the aggregator. 
In this work, we provide a noise-aware probabilistic modeling framework, which allows Bayesian inference to take into account the noise added for privacy under LDP, conditioned on locally perturbed observations.
Stronger privacy protection (compared to the central model) provided by LDP protocols comes at a much harsher privacy-utility trade-off. Our framework tackles several computational and statistical challenges posed by LDP for accurate uncertainty quantification under Bayesian settings. We demonstrate the efficacy of our framework in parameter estimation for univariate and multi-variate distributions as well as logistic and linear regression.
\end{abstract}

\section{Introduction}

In many practical settings, researchers only have access to small and sensitive data sets. In order to prepare against
overconfident incorrect findings in these settings, properly assessing the uncertainty due to finiteness of the sample becomes a
crucial component of the inference. Besides the uncertainty estimation for which there is a wide range of tools in the Bayesian
literature, also the privacy of the data subjects should be preserved.

\par In recent years, local differential privacy (LDP)~\citep{EGS:03,KLNRS:08} has become the gold-standard for inference
under untrusted aggregation scenarios, enabling strong client-side privacy protection. Despite rich history of investigation and large-scale deployments, LDP has been rarely paired with Bayesian inference. 

\par In this work, we initiate the study of developing methods for Bayesian inference under LDP constraints, and quantify the uncertainty as a function of the scale of the noise injected. We focus on the \emph{non-interactive} setting, where each client participates only in a single round of data collection.

\par A generic solution for performing Bayesian inference under LDP would be to try enforcing LDP in a DP variant of a general-purpose Markov chain Monte Carlo (MCMC) algorithm e.g.~DP-SGLD~\citep{WYFS:2015,LCLC:19}. This is not an attractive solution for the following reasons: the iterative nature of these sampling algorithms demands multiple rounds of LDP-compliant data collection. Next, the cost of client-side implementation of a new computationally demanding sampling algorithm could be hard to justify when it is possible to collect sufficiently accurate information for other methods with a lightweight implementation. Additionally, the privacy cost scales proportionally to the number of iterations, which in the case of SGLD is equivalent to the number of posterior samples.

\par In contrast to the centralized DP setting, under LDP each input or sufficient statistic is perturbed at client side, and the aggregator is forced to infer the parameters only with access to the privatized inputs. These constraints suggest that for a non-interactive setting, similarly to the frequentist counterparts, Bayesian inference should be decoupled from data collection and conditioned directly on the privatized inputs. 

\cite{BS:18} proposed a sufficient statistics based approach to correctly quantify the posterior uncertainty by making the model
\emph{noise-aware}, i.e. by including the (centralized) DP noise mechanism into the probabilistic model.

\par Our solution extends the applicability of noise-aware models to LDP.  The sheer magnitude of noise due to LDP ($\Omega(\frac{\sqrt{N}}{\epsilon})$ for a sample size of $N$, compared to $O(\frac{1}{\epsilon})$ in central DP) makes our LDP inference problem much more challenging than in \cite{BS:18}'s centralized solution. Under central DP, the earlier works are based on perturbed sum of sufficient statistics and therefore the latent data appear in the model as an aggregate over the individuals. In contrast, under LDP the number of latent variables (the true unperturbed data) grows linearly with the number of samples, since the aggregator can only observe perturbed data.

\subsection{Related work}
Following \cite{WM:10}'s pioneering work, several approaches that combine differential privacy (DP) and Bayesian inference
\citep{PFCW:16,HJDH:19,HDNDK:18,JDH:17,FGJWC:16,WYFS:2015} have been proposed. These works demonstrate that it is possible to
adapt Bayesian learning techniques to obey the privacy constraints of DP, but they do not attempt to address the
additional uncertainty the DP mechanism induces to the inference.

\cite{BS:18} showed how to make a model \emph{noise-aware} by including the DP noise mechanism as a part of the probabilistic
model describing the data generating process in a statistical inference task. Later, these techniques were applied for linear
regression models by \cite{BS:19} and for generalized linear models by \cite{KJKKH:21}. These works are based on centralized DP
which assumes the presence of a trusted curator to perturb and release the sensitive computation. 

Motivated by the success of large scale industrial deployments \citep{snap:18,DKY:17,appleLDP:17,rappor}, the topic of LDP observed a massive growth in the last 7 years. A large body of research has mostly focused on solving problems such as estimating \emph{uni-dimensional distributions} \citep{LWLLS:20,JKMW:19,WXYZHSSYL:19,NXYSSS:16}, \emph{histograms and joint distributions} \citep{YWLCS:21,KCS:19,ZWLHC:18,KCS:18,WBLJ:17,rappor2:16}, \emph{heavy hitters} \citep{BNST:20,BS:15} to \emph{empirical risk minimization} \citep{DSX:20,DX:2019,WGX1:2018,STU:2017}. For a more comprehensive review, we point readers to the surveys by \cite{XLLCN:20,WZFY:20}.

\par The work closest to ours is from \cite{SWSZW:19}. They perform Bayesian inference for Poisson factorization models under a weaker variant of LDP. While there is some overlap in the way we handle latent variables, we consider a completely different set of problems, i.e.  parameter estimation of univariate and multivariate distributions and linear and logistic regression, which have not been studied beyond the point estimators, despite such a great progress in LDP. 
\subsection{Contributions}
Our work makes the following contributions.
\begin{enumerate}
    \item We propose a probabilistic framework for Bayesian inference under LDP, which captures the entire data generating process including noise addition for LDP. Use of our methods does not require any changes in the existing LDP data collection protocols since these run as a downstream task at aggregator's end. 
    \item We avoid the explosion in the model complexity due to latent variables, by marginalizing them out from the likelihood calculations. (See Section \ref{sec:marg_input})
    \item Unlike most prior works, our $\epsilon$-DP parameter estimation model for Gaussian data learns both the mean and the variance in a single round. (See Section \ref{Sec:unidimensional_parameter_estimation})
    \item We prove a new tight $(\epsilon,\delta)$-DP bound for releasing the sufficient statistics for linear regression. (See Lemma \ref{lemma:sensitivity_linear_regression})
    \item We provide a sufficient condition for the tractability of the integral marginalizing the latent input, under Gaussian perturbation. (See Lemma \ref{lemma:suff_condition_marginalization})
    \item With extensive experiments, we demonstrate that our private posteriors are well-calibrated and the posterior inference can provide higher utility than private (point estimate) baselines in strong privacy/low sample size  cases.
\end{enumerate}

\section{Background and Problem formulation}
\subsection{Local differential privacy (LDP)}
\label{sec:DP}

The classic centralized model of differential privacy presupposes the presence of a trusted aggregator that processes the private information of individuals and releases a noisy version of the computation. The local model instead safeguards user's inputs by considering the setting where the aggregator may be untrustworthy. Then, the user needs to add noise locally before sharing with the aggregator. Consider two users each having an abstract data points, $\bx$ and $\bx'$, from a domain $\mathcal{D}$.
\begin{definition}
For $\epsilon \geq 0, \delta \ge 0$, a randomized mechanism $\mathcal{M}$ satisfies $(\epsilon,\delta)$-local differential privacy~\citep{KLNRS:08} if for any two points $\bx, \bx' \in \mathcal{D}$, and for all outputs $Z \subseteq Range(\mathcal{M})$, the following constraint holds: 
\begin{equation}\label{eq:dp}
    \Pr[\mathcal{M}(\bx) \in Z] \leq  \exp(\epsilon) \times \Pr[\mathcal{M}(\bx') \in Z]  + \delta.
\end{equation}
\end{definition}
Lower value of $\epsilon$ and $\delta$ provides a stronger protection of privacy. When $\delta=0$, $\mathcal{M}$ is
said to satisfy pure-LDP or $\epsilon$-LDP. 

Among many desirable properties of a privacy definition, (L)DP degrades gracefully under repeated use and is immune to
post-processing. The latter means that the privacy loss of $\mathcal{M}$ cannot be increased by applying any randomized
function independent of the data to $\mathcal{M}$’s output.

For practical implementations of differential privacy, we need to quantify the worst-case impact of an individual's record on
the output of a function. This quantity is referred to as \emph{sensitivity} and is defined as follows:
\begin{definition} For all $\bx, \bx' \in \mathcal{D}$, the $L_p$ sensitivity of a function $\bt:\mathbb{R}^{d} \rightarrow
\mathbb{R}^{m}$ is defined as 
%over all $\bX$'s and its neighbors $\bX'$'s 
$$
\Delta_{p}(\bt) = \max_{\bx, \bx' \in \mathcal{D}}|| \bt(\bx) - \bt(\bx')||_{p}.
$$
\end{definition}
%\tejas{The post-processing property is not yet defined but referred elsewhere.}

We now review some basic LDP perturbation primitives that will be referred to in the paper.

\para{Analytic Gaussian Mechanism~\citep{analyticGaussian}}. Classical Gaussian mechanism satisfies $(\epsilon,\delta)$-DP when $\epsilon \in (0,1)$. Much higher $\epsilon$ values are commonly used in practice in LDP scenarios. 
\cite{analyticGaussian} proposed an algorithmic noise calibration strategy based on the Gaussian cumulative density function (CDF) to obtain a mechanism that adds the least amount of Gaussian noise needed for $(\epsilon, \delta)$-DP:

\begin{definition}{(Analytic Gaussian Mechanism)} For any $\epsilon \geq 0, \delta \in [0,1]$, a mechanism $\mathcal{M}(\bD)$ which releases  $\bZ = \bt(\bD) + \bzeta$ with sensitivity $\Delta_{2}(\bt) =: \Delta_{2}$, satisfies $(\epsilon,\delta)$-DP with $\bzeta \sim \mathcal{N}(\bzero,\bI\sigma_{*}^{2} )$ iff
\begin{equation} \label{eq:agm}
\begin{aligned} 
    \Phi\left( \frac{\Delta_{2}}{2\sigma_{*}}- \frac{\epsilon \sigma_{*}}{\Delta_{2}} \right) -  \exp(\epsilon)\Phi\left(- \frac{\Delta_{2}}{2\sigma_{*}}- \frac{\epsilon \sigma_{*}}{\Delta_{2}} \right)  \leq \delta.
\end{aligned}
\end{equation}
\end{definition}

We now recall two $\epsilon$-DP primitives for releasing uni-variate inputs. 

\para{Laplace Mechanism~\citep{DworkMNS06}.}

\begin{definition}
For any $\epsilon \geq 0$, the Laplace mechanism $\mathcal{M}(\mathcal{D})$ releases $Z=t(\mathcal{D})+\zeta$, where $\zeta$ is sampled from a zero-mean Laplace distribution with the scale parameter $\frac{\Delta_{1}(t)}{\epsilon}$.
\end{definition}

\para{Randomized Response (RR)~\citep{Warner:65}.}  
The 1-bit randomized response is a utility-wise optimal \citep{CSS:12,KSV:16} way of releasing the sum of bits under LDP.
When $ x \in \{0, 1\}$, a user satisfies $\epsilon$-DP by reporting the perturbed input $z$ from the following distribution: 
$$  z=
\begin{cases}
x , \text{with probability }  p=\frac{\exp(\epsilon)}{1+\exp(\epsilon)} > \frac{1}{2},  \\
1-x, \text{with probability }  1-p. \\
\end{cases}
$$  
The unbiased point estimate for the mean of a sample of size $N$ can be recovered by $\frac{\frac{1}{N}\sum_{i \in[N]} z_i +p-1}{2p-1}.$

\subsection{Our setting and goal} In this work, we assume the \emph{non-interactive} untrusted aggregation setting, commonly considered in LDP works. We have $N$ non-colluding users, each user $i \in [N]$ holding a private data point $\bx_{i}$. We model the $\{\bx_i\}_{i=1}^N$ as an independent and identically distributed (iid) sample from a probability distribution with density function $f$ parameterized by $\btheta$.

\par Each user $i \in [N]$ locally perturbs $\bx_{i}$ to obtain their privatized version $\bz_{i}$ using a LDP compliant mechanism $\mathcal{M}$ and shares it with the aggregator. Using $\bZ=\{\bz_{1},..,\bz_{N}\}$ and the knowledge of $\mathcal{M}$, the aggregator intends to recover the posterior distribution $\Pr[\btheta|\bZ]$.

We emphasize that from the DP perspective the posterior inference is a downstream task and
can performed with any off-the-shelf sampling algorithm. We pay a privacy cost only once for
perturbing the inputs, and obtain the posterior distribution without further privacy cost as
a consequence of the post-processing property of (L)DP.

In the next two sections, we propose two simple ways of modeling single-round LDP-compliant data collection protocols. 

\section{Probabilistic model for noise-aware inference}
\subsection{Modeling with sufficient statistics}
\label{Sec:sum_of_sums}
For certain statistical models, we can characterize the distribution of observations by a
finite number of sufficient statistics. For such models we can deploy the LDP mechanism on
the individual sufficient statistics. This means that there exists a mapping 
$\bt : \{\bx_i\}_{i=1}^N \rightarrow \bs$ such that
\begin{align*}
    \Pr[\btheta | \bx_1,\dots,\bx_N] = \Pr[\btheta | \bs],
\end{align*}
i.e.\ $\bs$ contains all the information of $\{\bx_i\}_{i=1}^N$ relevant for the probabilistic inference task.
Denoting the latent unperturbed sufficient statistics with $\bS$ and the perturbed sum with $\bZ$, our
probabilistic model becomes
\begin{align*}
    \Pr[\btheta, \bS, \bZ] = \Pr[\btheta] \Pr[\bZ \mid \bS]  \Pr[\bS\mid \btheta],
\end{align*}
where $\btheta$ are the model parameters.

\subsubsection{Linear regression} 
For linear regression, the sufficient statistics $\bs \in \mathbb{R}^{d+2 \choose d}$ for each input
$\{\bx,y\}$ are 
\begin{equation*}
    \bs=\bt(\bx,y)= [\text{vec}(\bx \bx^{T}), \bx y,y^{2}].
\end{equation*}
We assume that each individual perturbs these sufficient statistics locally using the Gaussian mechanism. Next, we split the 
sum of perturbed sufficient statistics $\bZ$ as the sum of true sufficient statistics $\bS$ and the sum of local perturbations
($\sum_{i} \bzeta_i$). Conveniently for Gaussian noise, the sum of perturbations is also Gaussian with the variance 
scaled up by $N$.  Note, that for other noise mechanisms that are not closed under summation (e.g. Laplace mechanism), 
we could still approximate the sum of noise as a Gaussian using the central limit theorem.

Using a normal approximation to model the latent sum of sufficient statistics $\bS$ and Gaussian perturbation as the noise
mechanism, we can marginalize out the latent $\bS$ and obtain the following posterior for the model parameters 
(full derivation in Section~\ref{sec:posterior_calculations_GLM} in the supplementary article):
\begin{align}
    \Pr[\btheta, \bSigma| \bZ] \propto  \Pr[\btheta] \Pr[\bSigma]
        \mathcal{N}(\bZ;\bmu_{s},\bSigma_{s}+N\bSigma_{*}),
    \label{eq:priv_model_2} 
\end{align}
where $\bSigma_{*} \in \mathbb{R}^{{d+2 \choose d} \times {d+2 \choose d}}$ is the diagonal covariance  matrix
of Gaussian noise, and  $\bmu_{s} \in \mathbb{R}^{d+2 \choose d},\bSigma_{s} \in \mathbb{R}^{{d+2 \choose d}
\times {d+2 \choose d}}$ are the mean and covariance of $\bS$. This model is similar to \cite{BS:19}'s model that considered this problem in the centralized setting, and we use the closed form expressions for
$\bmu_{s}$ and $\bSigma_{s}$ from their work. 

\para{Our contribution}. 
\cite{BS:19}'s solution for linear regression satisfies $\epsilon$-DP, but with a sensitivity expression depending
quadratically on $d$ because they assume bounds on the individual components of $\bx$. Moreover, they analyze the sensitivities for the three components of $\bs$ separately. Adding noise with a scale proportional to $d$ is likely to result in so high level of noise under LDP that any signal in the data becomes indistinguishable. Our main contribution for this model is in the form a new bound for $\Delta_2(\bt)$ for linear regression with no dependence on $d$. We show in Lemma~\ref{lemma:sensitivity_linear_regression} that analyzing all three components of $\bs$ together instead of treating them separately as in earlier works, leads to a tighter bound on $\Delta_2(\bt)$. Towards this goal, we define the following convenience functions:
\begin{align*}
             \bt_1(\bx) &= \bx, \\
        \bt_2(\bx) &= \begin{bmatrix} x_1^2,  \ldots , x_d^2, 
          x_1x_2, \ldots ,  x_{d-1} x_d \end{bmatrix}^T. 
\end{align*}

Using above functions,  $\bt(\bx,y)=[\bt_2(\bx), y\bt_1(\bx),y^2]$. 

\begin{lemma}
\label{lemma:sensitivity_linear_regression}
Assume $|| \bx ||_2 \leq R$ , $|y| \leq R_y$, and let $\bt_1$ and $\bt_2$ be defined as above and let  $\sigma_1,\sigma_2,\sigma_3>0$.
Consider the Gaussian mechanism 
$$
\mathcal{M}(\bx) = \begin{bmatrix} \bt_2(\bx) \\ y \bt_1(\bx) \\ y^2  \end{bmatrix} + 
\mathcal{N}\left(0, \begin{bmatrix} \sigma_1^2 \bI_{d_2} & 0 & 0 \\ 0 & \sigma_2^2 \bI_{d} & 0 \\ 0  & 0 & \sigma_3^2 \end{bmatrix}   \right),
$$
where $d_2 = {d + 2 \choose 2}$. 
The tight $(\epsilon,\delta)$-DP for $\mathcal{M}$ 
is obtained by considering a Gaussian mechanism with noise variance $\sigma_1^2$ and sensitivity
$$
\Delta_2(\bt)  \leq \sqrt{\frac{\sigma_1^{4} R_{y}^4 }{2\sigma_2^{4}} + 2R^4+ \frac{2\sigma_1^2 R_{y}^2 R^2}{\sigma_2^2}  + \frac{\sigma_1^2 R_y^{4}}{\sigma_3^2}}.  
$$
\end{lemma}
\begin{proof}
Proof can be found in Section~\ref{proof:sensitivity_linear_regression} in the Supplement.
\end{proof}
  \begin{corollary}
\label{cor:linear_regression}
In the special case $R=1$ and $\sigma_1 = \sigma_2 =\sigma_3= \sigma$, by Lemma~\ref{lemma:sensitivity_linear_regression}, the optimal $(\epsilon,\delta)$ is obtained by considering the Gaussian mechanism
with noise variance $\sigma^2$ and sensitivity $\Delta_2(\bt) \leq \sqrt{\frac{ 3R_{y}^4}{2}+2R_{y}^{2} + 2}$.
\end{corollary}

\subsubsection{Logistic regression}
\label{sec:logistic_regression_ss}
Analogously, we can derive an approximate sufficient statistics based logistic regression model. We use the privacy results and  the calculations for $\bmu_{\bs}, \bSigma_{\bs}$ from \cite{KJKKH:21}. The privacy results are summarized in Section~\ref{sec:sufficient_statistics_lr} in the Supplement.

\subsection{Modeling with perturbed inputs}
\label{sec:marg_input}
To generalize our framework beyond models with sufficient statistics, we consider the probabilistic model
\begin{align}
    \Pr[\bz,\bx \mid \btheta] = \Pr[\bz \mid \bx] \Pr[\bx \mid \btheta],
\end{align}
where $\bz$ denotes the perturbed observation of the latent input $\bx$. This formulation allows us to work with arbitrary
input distributions $\Pr[\bx \mid \btheta]$ and $\Pr[\bz \mid \bx]$. However, introducing $N$ latent variables $\bx_i$, one
for each input, quickly makes the inference computationally infeasible as $N$ increases. To overcome this,  we marginalize out
the latent $\bx$:
\begin{equation}
	\begin{aligned}
    \Pr[\bz\mid \btheta] &=\int_{\Omega(\bx)} \Pr[\bz,\bx \mid \btheta] \dd\bx. 
    \end{aligned}
\end{equation}

As another benefit of working with the perturbed inputs, we do not need to presuppose any downstream usage of the data,
contrary to the sufficient statistic based models. This means that the privacy cost is paid once in the data collection, and
the aggregator can then use the data for arbitrary inference tasks. 

In the remainder of this Section, we exemplify how to marginalize the likelihood for different inference problems. We will highlight different
types of subproblems within each example (e.g.~modeling  clipping and enforcing the parameter constraints in the model), and demonstrate how
to solve them. Note that these solutions extend to a broad class of inference tasks, much beyond the scope of these examples. All derivations are 
in the Supplement. 

\subsubsection{Unidimensional parameter estimations}
\label{Sec:unidimensional_parameter_estimation}
We first demonstrate the marginalization in a typical statistical inference task, where we seek to find the parameters 
of a 1-dimensional generative process (e.g. parameters of Gaussian, Poisson, geometric, and exponential distribution).

Many of such distributions have unbounded support. The work from \cite{BS:18} suggests to specify a bound $[a,b]$ on the data
domain and discard all points outside it. This not only hurts the accuracy but also burdens the LDP protocol to spend budget
to mask non-participation. Instead, we have each user map their input falling outside this bound to the bound before
privatizing. For example, for Gaussian distribution, if $x \not\in [a,b]$, the user clips their input to $a$ if $x \leq a$
or to $b$ if $x \geq b$. We adhere to our probabilistic treatment of data, and model the clipped perturbed observations
using a \emph{rectified} probability density function 
\begin{equation}
    \begin{aligned}
    \Pr[x \mid \theta; a,b] = \delta_a (&x)\Pr[x \leq a] \\
                                        &+ \mathbf{I}_{[a,b]}(x)\Pr[x \mid \theta] + \delta_b(x) \Pr[x \geq b],
\end{aligned}
\end{equation}
where $\Pr[x \mid \theta]$ denotes the pdf of $x$ prior to clipping and $\delta$ the Dirac's delta function.
As a consequence of clipping, we observe peaks at $a$ and $b$ in the rectified density function.

\par In the Supplementary material (see Sections \ref{sec:Gaussian_distribution} and \ref{sec:Exponential_distribution}), we show
the marginalization for parameter estimation task of Gaussian and exponential observation models. In both of the cases we 
have used Laplace noise to satisfy pure $\epsilon$-DP.

\para{Sufficient condition for marginalization.} 
Consider that we observe $z$ by perturbing $x$ using Gaussian perturbation and that $x$ is assumed to be bounded within
the interval $(a,b)$. In general, we would like to evaluate the following integral to marginalize out the latent inputs $x$:
\begin{align}
    \label{eq:marg}
    \int_a^b \Pr[z,x] \dd x.
\end{align}
However, this integral is often intractable. The next result will present a general class of models for $x$ that allow the
marginalization in tractable form.

\begin{lemma}
    Assume the probabilistic model
    \begin{equation} \label{eq:t_12}
    \begin{aligned}
        x \sim p, \quad z \mid x &\sim \mathcal{N}(x, \sigma^2).
    \end{aligned}
    \end{equation}
    If $p$ is of the form $\Pr[x] = Cg(x)\exp(h(x))$, where $C$ is a normalization constant, $g$ and $h$ are polynomials and
    $h$ is at most a second order polynomial, then the integral marginalizing $x$ out of $\Pr[z,x]$ becomes tractable.
    \label{lemma:suff_condition_marginalization}
\end{lemma}

\begin{proof}
Proof can be found in  Section~\ref{proof:suff_condition_marginalization} in the Supplement.
\end{proof}

\subsubsection{Histogram aggregation}
\label{sec:histogram_model}
We now focus on the problem of histogram aggregation, which has enjoyed significant attention under LDP. We assume each user $i$ holds an item $x_i = k \in [d]$. Aggregator's aim is to estimate the normalized frequency $f_{k}$ of each item $k \in [d]$, $f_{k} = \frac{\sum_{i \in [N]} \mathbf{1}_{ x_i = k} }{N}.$ Among several approaches proposed, we pick  \cite{WBLJ:17}'s 1-bit randomized response method \emph{Optimal Unary Encoding} to specify the likelihood. This method satisfies $\epsilon$-DP while also achieving theoretically optimal reconstruction error.   

\para{Optimal Unary Encoding (OUE).} \cite{WBLJ:17} represent each input $x=k$ as a one-hot encoded vector $\be \in \{1,0\}^{d}$ with $e_{k}=1$ and $e_l =0, \forall l\neq k$. All users then apply 1-bit RR using probabilities $p$ and $q$ at each location $j\in [d]$ of $\be$ and obtain $z_j$ as follows:
\begin{align}
z_{j} =
\begin{cases}
1,  \text{ with prob. } q=\frac{1}{2}  \text{, if } e_{j}=1, \\
0,  \text{ with prob. } p=\frac{\exp(\epsilon)}  {1+\exp(\epsilon)}  \text{, if } e_{j}=0.
\end{cases}
\end{align}
Upon collecting all perturbed bit vectors $\bz_i \in \{0, 1\}^{d}, i \in [N]$, the aggregator computes the variance-minimizing point estimate $\hat{f}_k$ for item $k$ as $\hat{f}_{k}=\frac{\frac{1}{N}\sum_{i \in [N]}z_{i,k} -q}{p-q}.$ 

\par In a strong privacy setting, or when the true counts are low, $\hat{f}_{k}$'s can be negative or may not sum to 1.
\cite{WLLSL:20,CMM:21} present empirical comparisons of several of post-processing techniques that impose simplexity. 
The survey shows that the performance of these methods varies a lot across tasks, and none of the methods provides a general and 
principled way of mapping unconstrained outputs into the simplex.
%Many of these techniques may appear ad hoc, disconnected from the noise mechanism used, and their performance may be data set dependent. 

\par{\textbf{Our approach.}} We model the data generation process as a multinomial distribution parametrized by 
$\btheta \in \mathbb{S}^{K-1}$ with $\Pr[x = k | \btheta]= \theta_k, \forall k \in [K]$. Constraining $\btheta$'s simplexity in a model itself that describes data generation may improve accuracy, and eliminates the need of any further post-processing. The calculations for likelihood $\Pr[\bz,x=k|\btheta]$ can be found in Section~\ref{sec:OUE_likelihood} in the Supplement. 

%\subsubsection{Linear regression}
\subsubsection{Linear models}
\label{sec:linear_models}
Next consider a linear model where an outcome $y$ depends on variables $\bx$ as follows:
\begin{align}
    &\bx \mid \bSigma \sim \mathcal{N}(\mathbf{0}, \bSigma) \\
    &\mathbb{E}[y \mid \bx, \btheta] \sim g(\bx^{T}\btheta).
\end{align}
In Section \ref{sec:linear_regression_input_likelihood} of the Supplement, we show how to marginalize the inputs $(\bx, y)$
when $g(x) = x$ (linear regression) and both $X$ and $y$ are perturbed using Gaussian mechanism and in
\ref{sec:logistic_regression_input_likelihood} for $g(x)=1/(1+\exp(-x))$ (logistic regression) where $X$ is
perturbed with Gaussian mechanism and $y$ using the 1-bit RR.

\section{Experiments}

\subsection{Public data sets}
\label{Sec:data sets}
For histogram-related experiments on public data, we use Kosarak~\citep{kosarak}  data set which consists of 990K click
streams over 41K different pages. Our data set consists of $~11K$ records of randomly sampled $10$ items from the top-10K most frequent items.
\par For logistic regression, we use the Adult~\citep{Blake:98} data set (with over $48K$ records) from UCI repository to predict whether a person's income exceeds 50K. 
\par Finally, we train a linear regression model using Wine~\citep{CORTEZ2009547} data that comprises of roughly 6.5K 
samples with 12 features. We set $R_y=2$. A $80/20$\% train/test split was used.

\subsection{Private baseline --- LDP-SGD}\label{Sec:LDP-SGD}
We compare our regression models with a local version of DP-SGD adapted from \cite{WXYZHS:19} to our case study. They make the privacy cost invariant to the number of training iterations by having each user participate only once in the training. They partition the users into the groups of size $G=\Omega(\frac{d \log(d)}{\epsilon^{2}})$.  However, unlike in standard DP-SGD, the aggregator makes model updates for group $i$ only after receiving the noisy gradients from group $i-1$. The gradients are perturbed using Gaussian noise (with $\Delta_{2}=2$) after clipping their $L_2$ norm to 1. In total, $2N-G$ messages are exchanged in the protocol, which is roughly $2N$ for large enough $\epsilon$'s. 

\para{Pre-processing for regression tasks.} To reduce the training time to be more manageable, we use principal component analysis (PCA) to reduce the original dimensionality of our data sets. Additionally, we center our data sets to zero mean.

\subsection{Default settings and implementation}

Both noise-aware and non-private baseline models are implemented either in Stan~\citep{Stan} or NumPyro~\citep{PPJ:19,BCJO:18}. 
We use \emph{No-U-Turn}~\citep{nuts} sampler, which is a variant of Hamiltonian Monte Carlo. 
We run 4 Markov chains in parallel and discard the first 50\% as warm-up samples. We ensure that the Gelman-Rubin convergence statistic~\citep{R_hat:1998} for all models remains consistently below 1.1 for all experiments. For regression experiments, we fix $R=1$ to use Corollary~\ref{cor:linear_regression}  and ~\ref{cor:logistic_regression}.

\para{Prior selection.} The priors for model parameters are specified in Table~\ref{tab:priors}, and in Section~\ref{sec:LKJ} in the Supplement. The arguments to these priors are the hyper-parameters for the model.

\subsection{Uncertainty calibration} 
For the deployment of the learned models, it is crucial that the uncertainty estimates of the models are calibrated.
Overestimation of variance would lead to uninformative posteriors whereas underestimation can lead to overconfident estimates.
To assess this, we use synthetic data sampled from prior predictive distribution and learn the posterior distributions for the
model parameters. We repeat this multiple times, and count the fraction of runs that included the data generating parameters for
several central quantiles of the posterior samples.
Figure~\ref{fig:gaussian} shows the posterior calibration  for the Gaussian parameter estimation model as a function of $\epsilon$.
We can see that our posteriors are well calibrated even for the most strict privacy constraints. The plots for the histogram and exponential distribution models can be found in Figures~\ref{fig:OUE} and \ref{fig:exponential} in the supplement. 

\subsection{Accuracy of private posteriors}
We can also test the accuracy of our model by directly comparing the private and non-private posteriors.
Figure~\ref{fig:OUE_real_data} compares the mean absolute deviation over 50 repeats in the empirical cumulative density
function (ECDF) of the private posteriors of the multinomial probabilities from the non-private for the Kosarak data set. We verify that our private
posteriors, despite adhering to LDP guarantees, are quite close to the non-private posteriors. 

\par Additionally, to motivate the benefit of including the whole data generation process in a single model, we compare the posterior means of our histogram aggregation model to the private point estimates. As mentioned in Section~\ref{sec:histogram_model}, the denoised point estimates may not satisfy simplexity, especially for strong privacy parameters and/or low sample size. We can enforce simplexity in these estimates with additional post-processing e.g. by computing their least squares approximations~\citep{HRMS:10}. While the least squares solution is well principled, it is disconnected from the main LDP protocol. By modeling $\btheta$ as a multinomial distribution, we automatically satisfy $\btheta$'s simplexity in the model itself. This leads to better utility in some cases. 
Figure~\ref{fig:point_estimates_vs_posteriors_0.5} compares the mean squared error in reconstruction for a three dimensional histogram with true weights $[0.7,0.2,0.1]$ for both methods for $\epsilon=0.5$. We observe that the error in our model is lower by 10-30\%, specially in low sample size regimes. The comparison for more $\epsilon$ values can be found in Figure~ \ref{fig:point_estimates_vs_posteriors} in Supplement.

\subsection{Accuracy of inference}
Besides being able to correctly estimate the uncertainty of model parameters, we want our inference to also be accurate. Towards this goal, we compare the performance metrics of our regression models to LDP-SGD (Section~\ref{Sec:LDP-SGD}).
  
\para{Linear regression.} We trained the linear regression model using the Wine~\citep{CORTEZ2009547} data set. Figure
\ref{fig:linear_regression} shows that the sufficient statistics based Bayesian model performs well under strict privacy
guarantees. The comparison method (LDP-SGD) starts to perform better when we increase the privacy budget. As the sufficient statistic based model is trained on clipped inputs, it may underestimate the test targets which are unclipped, thus hurting the RMSE. 

\para{Logistic regression.} We run the sufficient statistic based model for all 48K records of the adult data set. However, for the input based model (Section~\ref{sec:linear_models}), we use randomly sampled $10K/48K$ records to reduce the training time. We choose $c=0.5$ because this split yielded the best utility in our internal experiments (not included). For a fair comparison, we also run LDP-SGD with the same 10K records.

\par Figure~\ref{fig:logistic_regression_ss_input} compares the mean AUC for LDP-SGD, and both models for various privacy parameters. We verify that for $\epsilon \leq 0.8$, the sufficient statistics based Bayesian model outperforms DP-SGD. For large enough $\epsilon$'s,  sufficient statistics based Bayesian model achieves nearly the same level of performance as DP-SGD without additional rounds of messages from aggregator. The input based model cannot, however, perform at a similar level. This is possibly caused by the large amount of noise (due to the budget split), that suppress the covariance information necessary for accurate inference. In contrast, the covariance structure is perturbed relatively cheaply in the sufficient statistics based solution due to tight privacy accounting. We exclude the input based solution from our next plot.

\begin{figure}
    \centering
    \includegraphics[width=1.\linewidth]{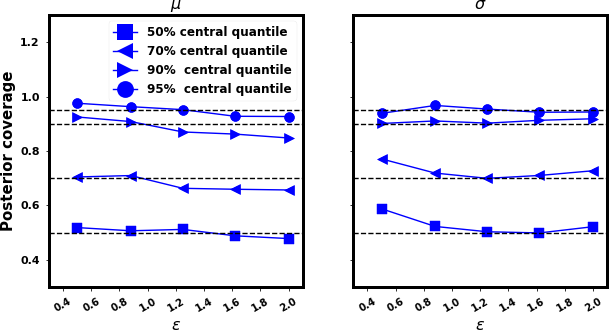} 
    \caption{Gaussian parameter estimation: Even under strict privacy guarantees, the posteriors are well calibrated. Figure shows posterior coverage calculated from $5000$ Gaussian samples (clipped to $[-5,5]$) over 500 repeats.    
    }
    \label{fig:gaussian}
\end{figure}

\begin{figure}
    \centering
    \includegraphics[width=1.\linewidth]{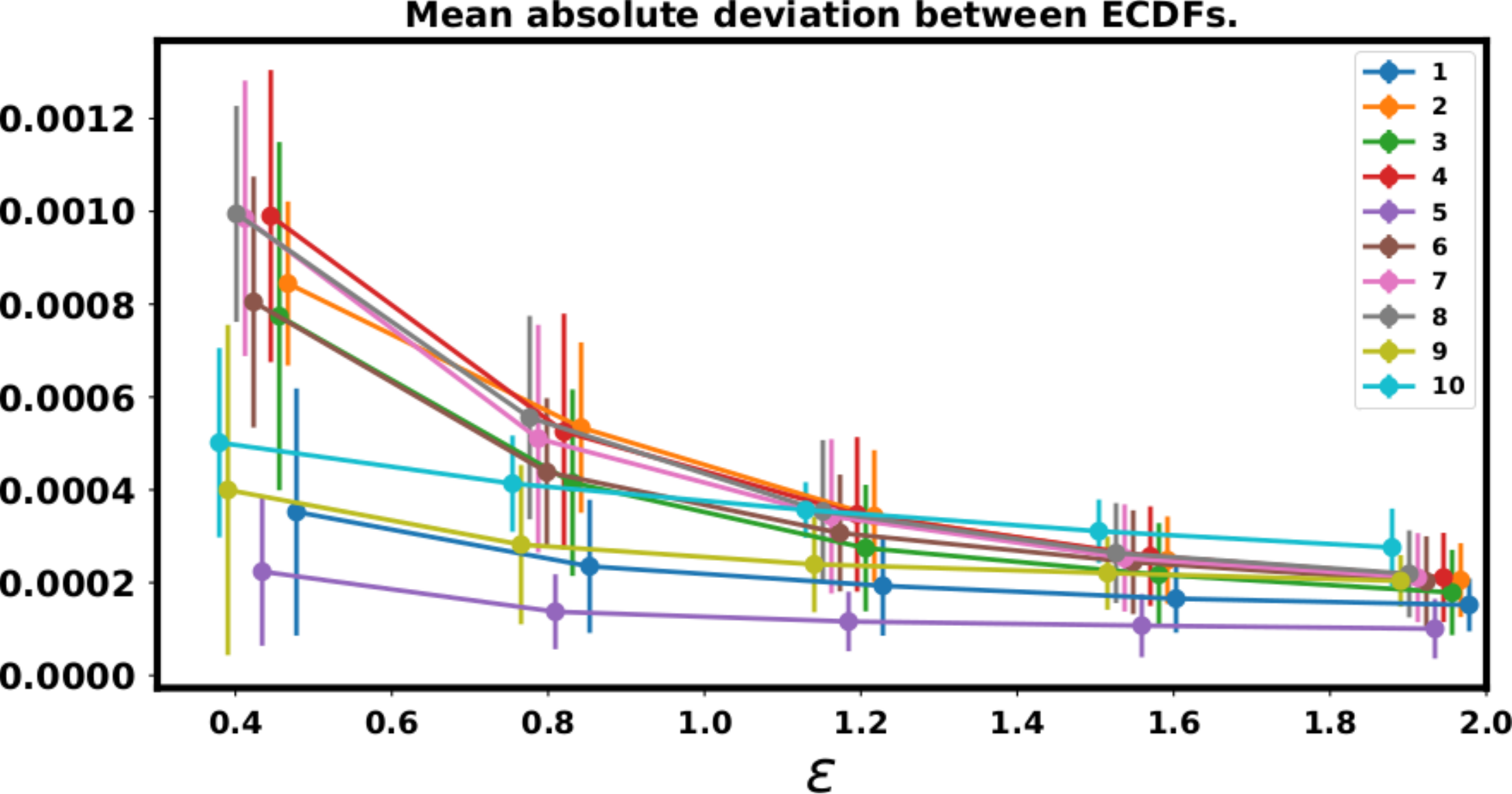}
    \caption{Histogram discovery: The LDP variant performs close to non-private even in the strict privacy regime and converges to the non-private posteriors. The figure shows the mean absolute difference over 50 repeats between the empirical CDFs of private and non-private posteriors for each of the 10 dimensions of an abridged version of the Kosarak~\citep{kosarak} data set.}
    \label{fig:OUE_real_data}
\end{figure}

\begin{figure}
    \centering
    \includegraphics[width=1.0\linewidth]{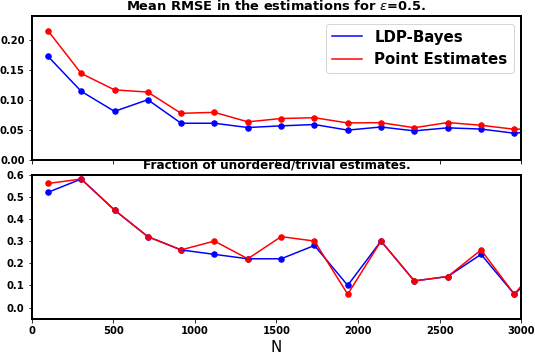}   
    \caption{Histogram estimation: The full-Bayesian version of the histogram aggregation method OUE \citep{WBLJ:17} outperforms the corresponding point estimates for small sample sizes. \textbf{Top:} the mean RMSE across 50 independent runs for LDP posterior means and $L_2$ approximations of LDP point estimates. The true histogram parameters ([0.7,0.2,0.1]) were used as the ground truth for computing RMSE. \textbf{Bottom:} the fraction of unusable or misleading (trivial, unordered) solutions produced by the Bayesian and the point estimates model.}
    \label{fig:point_estimates_vs_posteriors_0.5}   
\end{figure}

\begin{figure}
    \centering
    \includegraphics[width=\columnwidth]{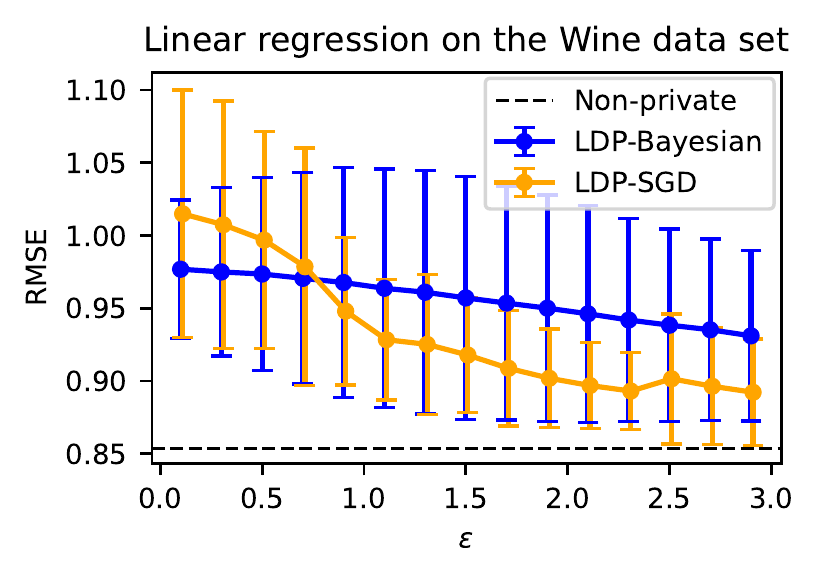}
    \caption{Linear regression: In the most important high-privacy region (small $\epsilon$) the sufficient statistics based Bayesian model outperforms the private baseline (LDP-SGD). However, as the $\epsilon$ increases the SGD based solution starts to perform better that the Bayesian approach. The results show the average test RMSE on Wine data set over 30 independent repeats of the inference and the errorbars denote the standard deviation among the repeats.\label{fig:linear_regression}}
\end{figure}

\begin{figure}
    \centering
    \includegraphics[width=\columnwidth]{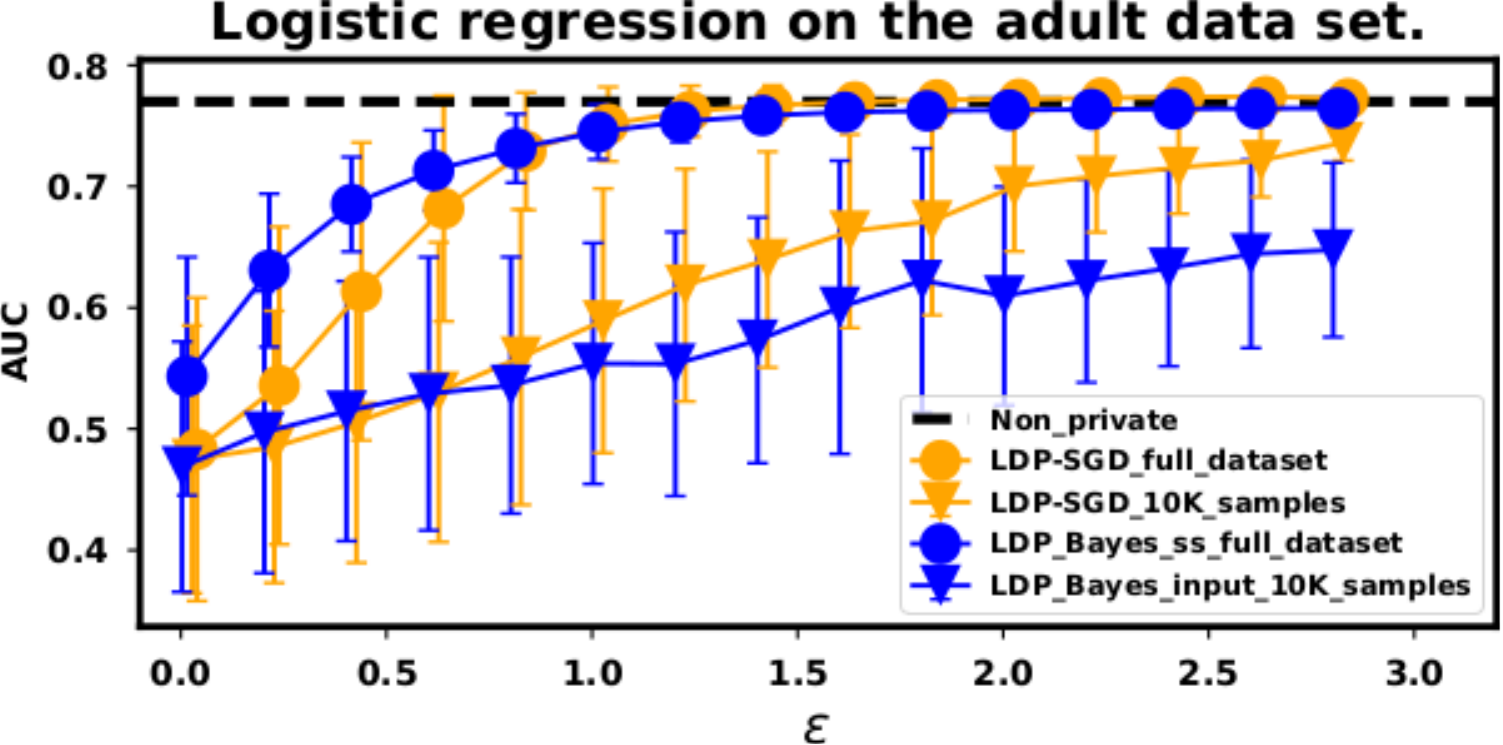}
    \caption{Logistic regression: The sufficient statistics based private logistic regression model almost matches the private and the non-private baselines as $\epsilon$ increases. The figure shows the average AUC for adult data set~\citep{Blake:98} for LDP-SGD and Bayesian models based on perturbed inputs and sufficient statistics. The mean is computed over 30 independent repeats of the inference. The errorbars are slightly shifted for readability and denote the standard deviation among the repeats.}
        \label{fig:logistic_regression_ss_input}
\end{figure}

\section{Discussion and concluding remarks}
In this work we have initiated the study of designing noise-aware models for performing Bayesian inference under LDP. Our models
are well calibrated and outperform the point estimations for small privacy/sample size regimes. 

\par With hierarchical modeling, these approaches easily extend to other interesting scenarios for pooled analysis between different sub-populations, such as personalized privacy
budgets \citep{CLQKJ16,JYC:15}, hybrid model \citep{AKZHL:19}, and multi-party differential privacy~\citep{Vadhan17} with minor modifications. Additionally, our models maintain full compatibility with more recent amplification techniques such as \emph{shuffling}~\citep{EFMRTT:19}.

\par Feature sub-sampling is routinely used in LDP protocols to amplify privacy and reduce communication. The current framework
does not efficiently model protocols involving sub-sampling because unsampled features in each input become the latent variables
of the model, thus exploding the model complexity. The question of modeling the uncertainty due to sub-sampling is left as a
future exercise. 

\section{Acknowledgement}
This work was supported by the Academy of Finland (Flagship programme: Finnish Center for Artificial Intelligence, FCAI) Grants 325572, 325573, the Strategic Research Council at the Academy of Finland (Grant 336032) as well as UKRI Turing AI World-Leading Researcher Fellowship, EP/W002973/1. We are grateful to the Aalto Science-IT project for their computational resources. 
\bibliography{papers}

\onecolumn
\aistatstitle{Locally Differentially Private Bayesian Inference}

\section{Chebyshev approximation}
 \label{sec:chebyshev_approx}
The polynomial expansion can be used to approximate a non-linear function by expressing it as a sum of monomials. Among several choices, Chebyshev polynomials \citep{mason2002chebyshev} are often chosen because the quality of approximation is uniform over any finite interval $[-R,R], R>0$.  
The Chebyshev polynomial expansion of degree $M$ for the sigmoid of $\bx^T \btheta$ is given below.
\begin{align}
 \frac{\exp(\bx^{T}\btheta)}{1+\exp(\bx^{T}\btheta)}  \approx \sum_{m=0}^{M} b_{m} (\bx^{T} \btheta)^{m} =\sum_{m=0}^{M} b_{m} \sum_{\be \in \mathbb{N}^{d}: \sum_{j} \be_{j}=m } {m \choose \be} (\btheta)^{\be}    ( \bx)^{\be}.     
\end{align}

In the above expression, $(\bx)^{\be}= \prod_{j \in [d]}x_j^{e_{j}}, \bx \in \mathbb{R}^{d},\be \in \mathbb{N}^{d}$. ${m \choose \be} = \frac{m!}{\prod_{j=1}^{d} e_j!}$ are the multinomial coefficients. The constants $b_{0},b_{1},\cdots,b_{M}$ are the Chebyshev coefficients computed over an arbitrary finite interval. In general, $M$ order approximation contains  ${d+M \choose d}$ summands. Figure~\ref{fig:ss} shows a 2nd order expansion of the sigmoid function.

\par{\textbf{Quality of approximation}}. We would like to empirically understand the quality of approximation for the expression $\mathbb{E}_{x \sim \mathcal{N}^{d}(\bzero,\bSigma)}( \frac{\exp(\bx^{T}\btheta)}{1+\exp(\bx^{T}\btheta)})$, for a fixed $\btheta \in \mathbb{R}^{d}$ and $\bSigma \in \mathbb{R}^{d\times d}$.
\begin{figure}[!ht]
  \centering
   \includegraphics[width=0.5\linewidth]{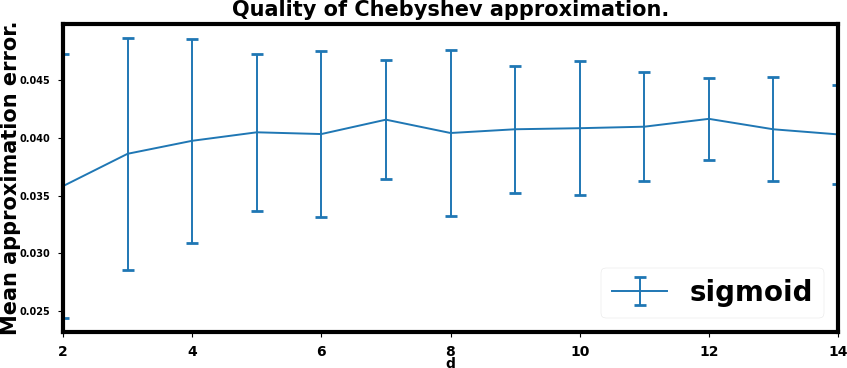}   
   \caption{We plot the mean absolute difference between empirical expectation of sigmoid($\bx^{T} \btheta$) and its 2nd order Chebyshev approximation with $b_0,b_1,b_2 \in [-3,3]$. The mean absolute deviation is computed over 50 independent runs. In each run, we sample a new $\btheta$ from standard normal distribution, and the empirical expectation is taken over $1000$ samples drawn from a zero centered multivariate Gaussian with a random full rank co-variance matrix. The co-variance matrix remains fixed for a given dimension for all runs. We also restrict the $L_2$ norm of these samples to 1. We can observe that the approximation is moderately accurate.}
         \label{fig:chebyshev_approx}
\end{figure}
\begin{figure*}
    \centering        
$\sigma(\bx^{T}\btheta) \approx b_0+ b_1x_1\theta_1+ b_1x_2\theta_2+b_1x_3\theta_3   + b_2x_1^{2}\theta^{2}_1 +b_2 x_2^{2}\theta^{2}_2+b_2x_3^{2}\theta^{2}_3  +   2b_2x_1 x_2 \theta_1\theta_2+2b_2x_1 x_3   \theta_1\theta_3 +2b_2x_2 x_3 \theta_2\theta_3$
\caption{An example of 2nd order Chebyshev expansion of the sigmoid function for $\bx, \btheta \in \mathbb{R}^{3}$.}
    \label{fig:ss}
\end{figure*}

\section{Priors for probabilistic models}
\begin{figure*}
\centering
\begin{tabular}{ |c|c|c| } 
\hline
Model & Prior   \\
\hline
Gaussian  & $\theta \sim \mathcal{N}(\mu_0,\sigma_0)$ and $\sigma \sim$ Gamma$(\alpha,\beta)$ \\ 
Exponential  & $\theta \sim$ Gamma$(\alpha,\beta)$   \\  
Multinomial & $\btheta \sim $Dirichlet($\{1\}^{d}$)  \\ 
Regression & $\btheta \sim \mathcal{N}(0,5\bI)$),  $\bSigma \sim$ scaled LKJ prior \\ 

\hline
\end{tabular}
\caption{Priors}
   \label{tab:priors}
\end{figure*}
\label{sec:LKJ}
\para{LKJ prior for $\bSigma$ \citep{lkj:09}}. We scale a positive definite correlation matrix from the LKJ correlation distribution of shape $\eta = 2$ from both sides with a diagonal matrix with $\mathcal{N}(0, 2.5)$ distributed diagonal entries. The probabilistic model is:
\begin{equation*}
\begin{aligned}
    \bOmega &\sim \text{LKJ}(2), \quad
    \boldsymbol{\tau} \sim \mathcal{N}(\mathbf{0}, 2.5 \cdot \mathbf{I}), \\
    \bSigma &= \text{diag}(\boldsymbol\tau) \, \Omega \, \text{diag}(\boldsymbol\tau).
\end{aligned}
\end{equation*}
\section{Missing Proofs}
\subsection{Proof of Lemma~\ref{lemma:sensitivity_linear_regression}}

\begin{proof}
\label{proof:sensitivity_linear_regression}
Let $\{\bx,y\},\{\bx',y'\}$ be the neighboring input pairs. We assume that $|| \bx ||_2 \leq R,|| \bx' ||_2 \leq R, |y|,|y'| \leq R_y$.
The Gaussian mechanism $\mathcal{M}$ for a single input $\{\bx,y\}$ is,
\begin{equation}
     \begin{aligned}
     \mathcal{M}(\bx)  &\sim \begin{bmatrix} \bt_2(\bx) \\ y\bt_1(\bx) \\ y^2  \end{bmatrix} + \mathcal{N}\left(0, \begin{bmatrix} \sigma_1^{2} \bI_{d_2} & 0 & 0 \\ 0 & \sigma_2^2 \bI_d & 0  \\
0 & 0 & \sigma_3^2   \\
 \end{bmatrix} \right) \\
  \mathcal{M}(\bx)  &  \sim \begin{bmatrix} \bI_d & 0 & 0 \\
  0 &  \frac{\sigma_2}{\sigma_1} \bI_{d_2} & 0  \\
  0 & 0 & \frac{\sigma_3}{\sigma_1}    \\
  \end{bmatrix} \left( \begin{bmatrix} \bt_2(\bx) \\ \frac{\sigma_1}{\sigma_2} y\bt_1(\bx) \\ \frac{\sigma_1}{\sigma_3} y^{2}  \end{bmatrix} +  \mathcal{N}(0, \sigma_1^2 \bI_{d_2+d+1} ) \right).
%   + %  \mathcal{N}\left(0, \begin{bmatrix} \sigma_1^2 I_d & 0 \\ 0 & \sigma_1^2 I_{d_2} \end{bmatrix} \right) \\
%      &  \sim \begin{bmatrix} I_d & 0 \\
%  0 &  \frac{\sigma_2}{\sigma_1} I_{d_2} 
%  \end{bmatrix} \left( \begin{bmatrix} t_1(\bx) \\ \frac{\sigma_1}{\sigma_2} t_2(\bx) \end{bmatrix} +   \mathcal{N}(0, \sigma_1^2 )  \right) \\
     \end{aligned}
 \end{equation}
Removing the constant scaling, we see that it is equivalent to the following mechanism.
\begin{equation*}
     \mathcal{M}(\bx) =  \begin{bmatrix} \bt_2(\bx) \\ \frac{\sigma_1}{\sigma_2} y\bt_1(\bx) \\ \frac{\sigma_1}{\sigma_3} y^{2}  \end{bmatrix}  +  \mathcal{N}(0, \sigma_1^2 \bI_{d_2+d+1} ).
 \end{equation*}

Define
$$
F(\bx,y) =  \begin{bmatrix} \bt_2(\bx) \\ \frac{\sigma_1}{\sigma_2} y\bt_1(\bx) \\ \frac{\sigma_1}{\sigma_3} y^{2}  \end{bmatrix}.
$$
 
 For the first order terms,
  \begin{equation} \label{eq:so}
   \begin{aligned}
 || y\bt_1(\bx) - y'\bt_1(\bx') ||^{2}_2  &= \sum_{j} (y x_j - y' x'_j)^{2} = y^2 \sum_{j} x^{2}_j + y'^2 \sum_{j} x'^{2}_j - 2yy' \sum_{j} x_j x'_j
\\    &\leq 2R^{2}_{y}R^2 - 2 yy'\langle \bx, \bx' \rangle.
  \end{aligned}
 \end{equation}

Similarly, $(y^{2}-y'^{2})^{2} \leq R_y^{4}$.
For the second order terms,
\begin{equation}
	\begin{aligned}
		\norm{\bt_2(\bx) - \bt_2(\bx')}_2^2 &= \sum\limits_{j} (x_j^2 - x_j'^2)^2 + 2 \sum\limits_{j>k} (x_j x_k -  x_j' x_k' )^2 \\
		&=  \sum\limits_{j} x_j^4 +  \sum\limits_{j} x_j'^4   - 2  \sum\limits_{j} x_j^2 x_j'^2 
		  + 2 \sum\limits_{j>k}  x_j^2 x_k^2  +   2 \sum\limits_{j>k}  x_j'^2 x_k'^2 - 4 \sum\limits_{j>k}  x_j x_k x_j' x_k' \\
		&=  \left( \sum\limits_{j} x_j^4 + 2 \sum\limits_{j>k}  x_j^2 x_k^2 \right) +  \left(  \sum\limits_{j} x_j'^4  +   2 \sum\limits_{j>k}  x_j'^2 x_k'^2 \right)
		 - 2  \sum\limits_{j} x_j^2 x_j'^2  - 4 \sum\limits_{j>k}  x_j x_k x_j' x_k'  \\
		&=  \left( \sum\limits_{j} x_j^4 + \sum\limits_{j \neq k}  x_j^2 x_k^2 \right) +  \left(  \sum\limits_{j} x_j'^4  +   \sum\limits_{j \neq k}  x_j'^2 x_k'^2 \right)
		 - 2  \sum\limits_{j} x_j^2 x_j'^2  - 2 \sum\limits_{j \neq k}  x_j x_k x_j' x_k'  \\
		&=  \norm{\bx}_2^4 +  \norm{\bx'}_2^4 - 2 \langle \bx, \bx' \rangle^2 \leq 2R^{4}- 2 \langle \bx, \bx' \rangle^2.
	\end{aligned}
\end{equation}

 %We make the factor $ 2R_y^{2} R^2 - 2yy' \langle \bx, \bx' \rangle $ as small as possible by choosing $y=y'=b$. 
Finally,
\begin{equation}
	\begin{aligned}
 \norm{F(\bx,y) - F(\bx',y')}^2_2  &= \norm{ \bt_2(\bx) - \bt_2(\bx') }^2_2 
+ \frac{\sigma_1^2}{\sigma_2^2} \norm{ y\bt_1(\bx) - y'\bt_1(\bx') }^2_2  + \frac{\sigma_1^2}{\sigma_3^2} (y^{2}- y'^{2})^2  \\ &\leq 2R^4 - 2 \langle \bx, \bx' \rangle^2 + \frac{\sigma_1^2}{\sigma_2^2} [2R_y^{2} R^2 - 2yy' \langle \bx, \bx' \rangle] + \frac{\sigma_1^2}{\sigma_3^2}R_y^{4} \\  &= 2R^4 - 2t^2+ 2c_1R_{y}^2  R^2  -2c_1 yy' t   +c_2R_y^{4}, \label{eq:quad}
	\end{aligned}
\end{equation}
 where $c_1 = \frac{\sigma_1^2}{\sigma_2^2}$, $c_2 = \frac{\sigma_1^2}{\sigma_3^2} $ and $t = \langle \bx, \bx' \big\rangle$. Above quadratic equation maximizes at $t=-\frac{c_1 yy'}{2}=-\frac{c_1 R_y^2}{2}$ and the maximum value is,
 \begin{equation}     
 \begin{aligned}
\Delta_2(\bt) = \norm{F(\bx,y) - F(\bx',y')}_2 \leq \sqrt{\frac{\sigma_1^{4} R_{y}^4 }{2\sigma_2^{4}} + 2R^4+ \frac{2\sigma_1^2 R_{y}^2 R^2}{\sigma_2^2}  + \frac{\sigma_1^2 R_y^{4}}{\sigma_3^2}}.    
 \end{aligned}
\end{equation}

 \end{proof}
%\subsection{Marginalizing the input for Gaussian noise model}

\subsection{Proof of Lemma~\ref{lemma:suff_condition_marginalization}}
\label{proof:suff_condition_marginalization}
\begin{proof}
    Let us write $g(x) = \sum_{j=0}^M g_j x^{j}$ and $h(x) = \sum_{i=0}^2 h_i x^{i}$. Now the integral in \eqref{eq:marg} becomes
    \begin{align}
        \Pr[z] = \int_a^b \Pr[z \mid x] \Pr[x] \dd x = C \frac{1}{\sqrt{2\pi \sigma^2}} \sum_{j=0}^{M} g_j \int_a^b x^j 
                    \exp{- \frac{1}{2\sigma^2} \left[(z-x)^2 - \sum_{i=0}^2 h_i x^i\right]} \dd x.
    \end{align}
    We can clearly see that the exponential term inside the integral is yet another Gaussian kernel of $x$. Next we write the 
    polynomial inside the exponential in a quadratic form:
\begin{equation} 
    \begin{aligned}
        (z-x)^2 - \sum_{i=0}^2 h_i x^i &= (1-h_2) x^2 - 2(z + \frac{h_1}{2}) x + z^2 - h_0 \\
        &=(1-h_2) \left[ 
                x^2 -2 \frac{(z + \frac{h_1}{2})}{(1-h_2)} x + \frac{(z+\frac{h_1}{2})^2}{(1-h_2)^2}
                - \frac{(z+\frac{h_1}{2})^2}{(1-h_2)^2}
            \right] + z^2 - h_0 \\
        &=(1-h_2)\left( x - \frac{z+\frac{h_1}{2}}{1-h_2} \right)^2 + z^2-h_0-\frac{(z+\frac{h_1}{2})^2}{1-h_2}.
\end{aligned}
\end{equation}
    Denote $s^2 = \frac{\sigma^2}{1-h_2}$ and $m = \frac{2z + h_1}{2-2h_2}$, we get
    \begin{align}
        \Pr[z] = C \frac{1}{\sqrt{2\pi \sigma^2}} \sqrt{2\pi s^2} \sum_{j=0}^{M} g_j \int_a^b x^j 
                   \mathcal{N}(x ; m, s^2) \exp(-\frac{1}{2\sigma^2}\left(z^2 - h_0 -\frac{(z+\frac{h_1}{2})^2}{1-h_2}\right))
                \dd x,
    \end{align}
    where $\mathcal{N}(\cdot; m, s^2)$ denotes the probability density function of a Gaussian with mean $m$ and variance $s^2$.
    Now, after removing the constant factors out of the integral, we are left with integral
    \begin{align}
        \int_a^b x^j \mathcal{N}(x ; m, s^2) \dd x = 
        \left( \Phi\left(\frac{b-m}{s}\right) - \Phi\left(\frac{a-m}{s}\right) \right) \mathbb{E}_{x\sim \text{TrunNorm(m,s,a,b)}}[x^j],
    \end{align}
    where $\Phi$ is the cumulative density function of a std. normal distribution. Now in order to conclude the proof, it suffices
    to show that the truncated normal distribution has the non-central moments in an tractable form, which has been shown in the past for example by \cite{flecher2010truncated}.
\end{proof}
\section{Likelihood calculations --- uni-variate parameter estimations.}
We denote $\Phi_{\btheta}(a) = \Pr[x \leq a]$ as the cumulative density function of the corresponding density function.
\subsection{Gaussian distribution}
\label{sec:Gaussian_distribution}
 We assume the distribution is rectified from both sides with $(a,b)
$ and data points are perturbed with $\epsilon$-DP Laplace noise. We intend to learn both mean $\mu$ and variance $\sigma^{2}$. Let $\btheta = [\mu,\sigma]$). 

\begin{align}
    \Pr[  z\mid \btheta] &= \int_{\mathbb{R}} \Pr[z, x \mid \btheta] dx  = \int_{\mathbb{R}} \Pr[z\mid x] \times \Pr[x \mid \btheta,a,b]  dx \\
    &= \frac{\epsilon}{2(b-a)} \int_{\mathbb{R}} \exp(- \frac{\epsilon|z-x|}{(b-a)}) \Pr[x \mid \btheta, a, b] dx \\ 
    &= \frac{\epsilon}{2(b-a)} \int_{\mathbb{R}} \exp(- \frac{\epsilon|z-x|}{b-a}) \Big[ \delta_a (x) \Phi_{\btheta}(a) + \mathbf{1}_{(a < x < b)} \Phi_{\btheta}(x) + \delta_b(x) (1-\Phi_{\btheta}(b)) \Big]dx \\
    &=\frac{\epsilon\Phi_{\btheta}(a)}{2(b-a)}\exp(-\frac{\epsilon|z-a|}{b-a}) + \frac{\epsilon(1-\Phi_{\btheta}(b))}{2(b-a)}\exp(-\frac{\epsilon|z-b|}{b-a}) \\ &+ 
    \frac{\epsilon}{2\sqrt{2\pi}(b-a) \sigma}\int_a^b \exp(- \frac{\epsilon|z-x|}{b-a})\exp(- \frac{(x-\mu)^2}{2\sigma^{2}})dx. \label{step:int}
\end{align}
Now we focus on the integral in step~\ref{step:int}. Since $x \in [a,b]$ and $z \in \mathbb{R}$, the sign of the term  $|z-x|$ depends on $z$. To simplify, we set limit $l= max(a,min(b,z))$.

\begin{equation} 
    \begin{aligned}
 & \int_{a}^{b} \exp(\frac{-\epsilon|x-z|}{(b-a)})\exp(- \frac{(x-\mu)^2}{2\sigma^{2}})dx \\ &=
    \int_{a}^{l} \exp(\frac{\epsilon(x-z)}{b-a})\exp(- \frac{(x-\mu)^2}{2\sigma^{2}})dx + \int_{l}^{b}
\exp(\frac{\epsilon(z-x)}{b-a})\exp(- \frac{(x-\mu)^2}{2\sigma^{2}})dx \\ &=
-\dfrac{\sqrt{{\pi}}{\sigma}\left(\operatorname{erf}\left(\frac{{\epsilon}{\sigma}^2+\left(b-a\right)\left({\mu}-l\right)}{\sqrt{2}\left(b-a\right){\sigma}}\right)-\operatorname{erf}\left(\frac{{\epsilon}{\sigma}^2+\left(b-a\right)\left({\mu}-a\right)}{\sqrt{2}\left(b-a\right){\sigma}}\right)\right)\exp(-\frac{{\epsilon}\left(2\left(b-a\right)\left(z-{\mu}\right)-{\epsilon}{\sigma}^2\right)}{2\left(b-a\right)^2})}{\sqrt{2}}
\\ &-\dfrac{\sqrt{{\pi}}{\sigma}\left(\operatorname{erf}\left(\frac{{\epsilon}{\sigma}^2-\left(b-a\right)\left({\mu}-l\right)}{\sqrt{2}\left(b-a\right){\sigma}}\right)-\operatorname{erf}\left(\frac{{\epsilon}{\sigma}^2-\left(b-a\right)\left({\mu}-b\right)}{\sqrt{2}\left(b-a\right){\sigma}}\right)\right)\mathrm{e}^\frac{{\epsilon}\left(2\left(b-a\right)\left(z-{\mu}\right)+{\epsilon}{\sigma}^2\right)}{2\left(b-a\right)^2}}{\sqrt{2}} 
\end{aligned}
\end{equation}
Putting this all together yields
\begin{equation} 
    \begin{aligned}
    \Pr[  z\mid \btheta] =  &
    \frac{\epsilon\Phi_{\btheta}(a)}{2(b-a)}\exp(-\frac{\epsilon|z-a|}{b-a}) +
    \frac{\epsilon(1-\Phi_{\btheta}(b))}{2(b-a)}\exp(-\frac{\epsilon|z-b|}{b-a})\\
    &+\frac{\epsilon}{4(b-a)}
    \mathrm{e}^{\frac{(\epsilon \sigma)^2}{2(b-a)^2}}
    \Bigg(
    \left(
    \operatorname{erf}\left(\frac{{\epsilon}{\sigma}^2+\left(b-a\right)\left({\mu}-a\right)}{\sqrt{2}\left(b-a\right){\sigma}}\right)-
    \operatorname{erf}\left(\frac{{\epsilon}{\sigma}^2+\left(b-a\right)\left({\mu}-l\right)}{\sqrt{2}\left(b-a\right){\sigma}}\right)
    \right)
    \exp(-\frac{{\epsilon}\left(z-{\mu}\right)}{(b-a)}) \\ 
    &+\left(
    \operatorname{erf}\left(\frac{{\epsilon}{\sigma}^2-\left(b-a\right)\left({\mu}-b\right)}{\sqrt{2}\left(b-a\right){\sigma}}\right)-
    \operatorname{erf}\left(\frac{{\epsilon}{\sigma}^2-\left(b-a\right)\left({\mu}-l\right)}{\sqrt{2}\left(b-a\right){\sigma}}\right)
    \right)
    \exp(\frac{{\epsilon}\left(z-{\mu}\right)}{(b-a)})
    \Bigg)
\end{aligned}
\end{equation}

\subsection{Exponential distribution}
\label{sec:Exponential_distribution}    
\begin{equation} 
    \begin{aligned}
     \Pr[ z \mid \theta,b ]  &= \int_{0}^{\infty} \Pr[z,x | \theta,b  ] dx =  \int_{0}^{\infty}   \Pr[z|x ] \times \Pr[x|\theta,b]  dx  \\ &=  \int_{0}^{\infty} \frac{\epsilon}{2b} \exp(\frac{-\epsilon|z-x|}{b}) \Big[  \mathbf{1}_{(x \leq b)} f(x)   + \delta_b(x) (1-\Phi_{\theta}(b)) \Big] dx \\ &=   \frac{\epsilon}{2b} \Big[ \int_{0}^{\infty} \exp(\frac{-\epsilon|z-x|}{b})  \theta \exp(-\theta x) dx + \int_{0}^{\infty}  \exp(\frac{-\epsilon|z-x|}{b}) \exp(-\theta b) \delta_{b}(x) dx \Big] 
\end{aligned}
\end{equation}
  Let us workout the first integration for the first case. The sign of $|z-x|$ may switch if $z <0$ and it may affect the outcome of the integral. So let's set the upper limit to a temporary limit $l$.  We set $l= max(0,min(b,z))$ to bound $z$ in $[0,b]$.

\begin{equation} 
    \begin{aligned}
& \int_{0}^{b} \exp(\frac{-|z-x|}{b/\epsilon})  \theta \exp(-\theta x) dx   =   \int_{0}^{l}\exp(\frac{x-z}{b/\epsilon})  \theta \exp(-\theta x) dx     +
         \int_{l}^{b}\exp(\frac{z-x}{b/\epsilon})  \theta \exp(-\theta x) dx \\ &=   
\dfrac{b \theta\left(\exp(l\theta)-\exp(\frac{l\epsilon}{b})\right)\exp(-\frac{\epsilon z}{b}-l\theta)}{b \theta- \epsilon} + \dfrac{b \theta\left(\exp(b\theta+ \epsilon)- \exp(l\theta+\frac{l\epsilon}{b})\right)\exp(\frac{\epsilon y-l\epsilon}{b}-b\theta-l\theta-\epsilon)}{b\theta+\epsilon}
\end{aligned}
\end{equation}

\section{Likelihood calculations for optimal unary encoding (OUE) }
\label{sec:OUE_likelihood}    
\begin{equation}
    \begin{aligned}
    \Pr[\bz, x=k \mid \btheta] &= \Pr[\bz \mid x=k] \Pr[x=k |\btheta]  \\ 
    &=\theta_k \times q^{\bz_{k}}(1-q)^{1-z_{k}} \prod_{j \neq k} p^{1-z_{j}} (1-p)^{z_{j}} \\ 
    &= \theta_k \times \left(\frac{1}{2}\right)^{z_{k}}\left(1-\frac{1}{2}\right)^{1-z_{k}} p^{\sum_{j \neq k} (1-z_{j})} (1-p)^{\sum_{j \neq k} z_{j} }
\end{aligned}
\end{equation}
Divide and multiply by $p^{1-z_{k}}(1-p)^{z_{k}}$.
\begin{align}
    & \Pr[\bz, x=k \mid \btheta] =  \theta_k  \frac{  p^{d- \sum_{j}z_{j}} (1-p)^{\sum_{j}z_{j}}}{2p^{1-z_{k}} (1-p)^{z_{k}}}
\end{align}
Let us marginalize the input $x$ out by summing over the domain.
\begin{equation}
    \begin{aligned}
    \Pr[\bz \mid \btheta] &= \sum_{k=1}^{d} \Pr[\bz,x=k \mid \theta] \\ 
    &= p^{d- \sum_{j}z_{j}} (1-p)^{\sum_{j}z_{j}} \sum_{k=1}^{d} \frac{\theta_k}{ 2p^{1-z_{k}} (1-p)^{z_{k}}}  \\ 
    &= p^{d} \left(\frac{1-p}{p}\right)^{\sum_{j} z_{j}} 
            \sum_{k=1}^{d} \frac{\theta_k}{ 2p } \left(\frac{1-p}{p}\right)^{-z_k} \\ 
    &= p^{d} \exp(-\epsilon \sum_{j} z_{j}) \sum_{k=1}^{d} \frac{\theta_k \exp( z_{k}\epsilon)}{2 p} 
\end{aligned}
\end{equation}
The last equality is due to the fact that $\frac{1-p}{p}= \exp(-\epsilon)$.

\section{Likelihood calculations for linear/logistic regression --- perturbed inputs}
Both models do not take any stand on how the sensitivity of the Gaussian mechanism was computed, or the way total privacy budget is split while perturbing $\bx$ and y.
\subsection{Linear regression}
\label{sec:linear_regression_input_likelihood}
Consider the following model where we observe data through Gaussian perturbation.

\begin{equation} 
    \begin{aligned}
    \bx &\sim \mathcal{N}(\mathbf{0}, \bSigma) \\
    \bz_{\bx} \mid \bx &\sim \mathcal{N}(\bx, \bSigma_*) \\
        y &\sim \mathcal{N}(\btheta^T\bx, \sigma^2) \\
    z_y  \mid y &\sim \mathcal{N}(y, \sigma_*^2).
    \end{aligned} 
\end{equation}
    
Next we show how to marginalize the latent inputs out from the model. We start by moving the $\bz_{\bx}$ out from the integral
\begin{align}
    \Pr[\bz_{\bx}, z_y] &= \int \int \Pr[\bx, \bz_{\bx}]\Pr[z_y, y \mid \bx] \dd\bx \dd y \\
    &= \mathcal{N}(\bz_{\bx}; \bzero, \bSigma + \bSigma_*)\int \int \Pr[\bx \mid \bz_{\bx}]\Pr[y \mid \bx]\Pr[\bz_y \mid y] \dd\bx \dd y \\
    &= \mathcal{N}(\bz_{\bx}; \bzero, \bSigma + \bSigma_*)\int \Pr[\bz_y \mid y]
        \int \mathcal{N}(\bx ; \underbrace{\bSigma(\bSigma + \bSigma_*)^{-1} \bz_{\bx}}_{:=\bh_{\bx}}, 
            \underbrace{\bSigma - \bSigma(\bSigma + \bSigma_*)^{-1}\bSigma}_{:=\bA})\Pr[y \mid \bx] 
                \dd\bx \dd y. \label{eq:linreg_marg_int}
\end{align}
Now note that since the $y$ depends on $\bx$ through $\btheta^T \bx$, the innermost integral can be written in terms of 
$\btheta^T \bx$, which in turn follows a Gaussian distribution $\mathcal{N}(\btheta^T h_{\bx}, \btheta^T\bA\btheta)$:
\begin{align}
    \int \mathcal{N}(\bx ; \bh_{\bx}, \bA)\Pr[(y \mid \bx] \dd\bx = 
        \mathbb{E}_{\bx \mid \bz_{\bx} \sim \mathcal{N}(\bh_{\bx}, \bA)}[\mathcal{N}(y; \btheta^T\bx, \sigma^2)] 
        &=
        \mathbb{E}_{\btheta^T\bx \mid \bz_{\bx} \sim \mathcal{N}(\btheta^T h_{\bx}, \btheta^T\bA\btheta)}
                        [\mathcal{N}(y; \btheta^T\bx, \sigma^2)] \\
        &=\mathcal{N}(y ; \btheta^T \bh_{\bx},  \btheta^T \bA \btheta + \sigma^2).
\end{align}
Finally, we substitute above into \eqref{eq:linreg_marg_int} and recover
\begin{align}
    \Pr[\bz_{\bx}, z_y] &= \mathcal{N}(\bz_{\bx}; \bzero, \bSigma + \bSigma_*)
        \int \Pr[\bz_y \mid y] \mathcal{N}(y ; \btheta^T h_{\bx},  \btheta^T \bA\btheta + \sigma^2) \dd y \\
        &= \mathcal{N}(\bz_{\bx}; \bzero, \bSigma + \bSigma_*) \mathcal{N}(\bz_y ; \btheta^T \bh_{\bx},  \btheta^T\bA\btheta + \sigma^2 + \sigma_*^2).
\end{align}

\subsection{Logistic regression}
\label{sec:logistic_regression_input_likelihood}
Consider the following model where we observe data through Gaussian and randomized response perturbation.
\begin{equation} 
    \begin{aligned}
    \bx &\sim \mathcal{N}(\mathbf{0}, \bSigma) \\
    \bz_{\bx} \mid \bx &\sim \mathcal{N}(\bx, \bSigma_*) \\
    y &\sim \bold{Bern}\Big(\frac{1}{1+\exp(-\bx ^{T} \btheta)}\Big) \\
    z_y  \mid y   &\sim \bold{Bern}(p).
    \end{aligned} 
\end{equation}

%Assume that $\{\bx,y\}$ is perturbed using Gaussian noise (with co-variance matrix $\bSigma_{*}$) and 1-bit RR with budgets $(1-c)\epsilon$ and $c\epsilon$ for  $c \in (0,1]$.

Next we show how to marginalize the latent inputs out from the model. %Assume $\bz_{\bx} | \bx \sim \mathcal{N}(\bx, \bSigma_{*})$ and $\bx \sim \mathcal{N}(\mathbf{0}, \bSigma)$.
Using the property A5 from \citep{sarkka:13}, we have
\begin{align*}
    \bz_{\bx} &\sim \mathcal{N}(\mathbf{0}, \bSigma_{*} + \bSigma)  \\
    \bx \mid \bz_{\bx} &\sim \mathcal{N}(\bSigma(\bSigma + \bSigma_{*})^{-1}\bz_{\bx}, \bSigma - \bSigma (\bSigma + \bSigma_{*})^{-1}\bSigma).
\end{align*}
We denote the posterior parameters with $\mathbf{h} = \bSigma(\bSigma + \bSigma_{*})^{-1}\bz_{\bx}$ and $\bA = \bSigma - \bSigma (\bSigma + \bSigma_{*})^{-1}\bSigma$. Now we can write

\begin{equation} 
    \begin{aligned}
 \Pr[\bz_{\bx},z_y|\btheta,\bSigma] &=    \int_{\Omega(\bX)  \Omega(\bY)}\Pr[\bz_{\bx},z_y,\bx,y|\btheta,\bSigma] d\Omega(\bX) d\Omega(\bY) \\ 
&= \int_{\Omega(\bX)  \Omega(\bY)}\Pr[\bx|\bz_{\bx}, \bSigma]  \Pr[z_y|y]\Pr[y |\bx,\btheta] \Pr[\bz_{\bx}|\bSigma] d\Omega(\bX) d\Omega(\bY) \\  
&= \Pr[\bz_{\bx} \mid \bSigma]\int_{\Omega(\bX)  \Omega(\bY)}\Pr[\bx|\bz_{\bx}, \bSigma]  \Pr[z_y|y]\Pr[y |\bx,\btheta]  d\Omega(\bX) d\Omega(\bY) \\  
&= \Pr[\bz_{\bx} \mid \bSigma]\int_{\Omega(\bX) }\Pr[\bx|\bz_{\bx}, \bSigma] \sum_{j \in \{1,0\}}p^{\mathbf{1}_{j=z_y}}(1-p)^{\mathbf{1}_{j\neq z_y}}\Pr[y=j | \bx,\btheta] d\Omega(\bX)\\  
&= \Pr[\bz_{\bx} \mid \bSigma]\int_{\Omega(\bX) }\Pr[\bx|\bz_{\bx}, \bSigma] \sum_{j \in \{1,0\}}p^{\mathbf{1}_{j=z_y}}(1-p)^{\mathbf{1}_{j\neq z_y}} \sigma(\bx^{T}\btheta)^{j}(1-\sigma(\bx^{T}\btheta))^{1-j} d\Omega(\bX) \\ 
&= \Pr[\bz_{\bx} \mid \bSigma]\int_{\Omega(\bX) }\Pr[\bx|\bz_{\bx}, \bSigma] \sum_{j \in \{1,0\}}(\frac{p}{1-p})^{\mathbf{1}_{j= z_y}} (1-p) \Big(\frac{\sigma(\bx^{T}\btheta)}{1-\sigma(\bx^{T}\btheta)} \Big)^{j}(1-\sigma(\bx^{T}\btheta)) d\Omega(\bX)\\ 
&= \frac{\Pr[\bz_{\bx} \mid \bSigma]}{{\exp(c\epsilon)+1}}\int_{\Omega(\bX) }  \sum_{j \in \{1,0\}}\exp(c\epsilon)^{\mathbf{1}_{j= z_y}}  \Pr[\bx|\bz_{\bx}, \bSigma] \Big[\frac{\exp(j\bx^{T}\btheta)}{1+\exp(\bx^{T}\btheta)}\Big] d\Omega(\bX) \\ &=
\Big( \frac{\mathcal{N}(\bz_{\bx};\bzero, \bSigma_{*}+\bSigma)}{{\exp(c\epsilon)+1}} \Big)\Big[ \exp(c\epsilon)^{\mathbf{1}_{z_y=1}} \mathbb{E}_{\bx \sim \mathcal{N}(\bh, \bA)}    \Big(\frac{\exp(\bx^{T}\btheta)}{1+\exp(\bx^{T}\btheta)}\Big)\\ &+ \exp(c\epsilon)^{\mathbf{1}_{z_y=0}} \mathbb{E}_{\bx \sim \mathcal{N}(\bh, \bA)}    \Big(\frac{1}{1+\exp(\bx^{T}\btheta)}\Big) \Big]  
\end{aligned}
\end{equation}

 To remove the non-linearity due to sigmoid, we employ Chebyshev approximation of order two. Section~\ref{sec:chebyshev_approx} introduces the Chebyshev approximation. 

\begin{align}
\label{eq:cheb}
\mathbb{E}_{\bx \sim \mathcal{N}(\bh, \bA)}    \Big(\frac{\exp(\bx^{T}\btheta)}{1+\exp(\bx^{T}\btheta)}\Big)  \approx b_0 + b_1\mathbb{E}_{\bx \sim \mathcal{N}^{d} (\bh,\bA)}[\bx^{T}\btheta] + b_2 \mathbb{E}_{\bx \sim \mathcal{N}^{d} (\bh,\bA)}[(\bx^{T}\btheta)^2]  
\end{align}
In the above expressions, $b_{0},b_1,b_2\in \mathbb{R}$ are the Chebyshev coefficients.
We can easily solve the expectations in \ref{eq:cheb} as below.

\begin{align*}
\mathbb{E}_{\bx \sim \mathcal{N}^{d}(\bh,\bA)}[x_j]&= h_j, \forall j \in [d] \\
\mathbb{E}_{\bx \sim \mathcal{N}^{d}(\bh,\bA)}[x_jx_k]  &=  A_{jk}  +h_jh_k, \forall j,k \in [d]   
\end{align*}
Similarly, we can approximate $\mathbb{E}_{\bx \sim \mathcal{N}(\bh, \bA)}    \Big(\frac{1}{1+\exp(\bx^{T}\btheta)}\Big)$ also.

\para{Sensitivity calculation for $\bx$.} Let  $\bx, \bx' \in \mathbb{R}^{d}, || \bx||_2\leq R,||  \bx'||_2 \leq R$ be the neighboring inputs. 
  \begin{equation} \label{eq:so}
   \begin{aligned}
 || \bx - \bx' ||^{2}_2  = \sum_{j} ( x_j -  x'_j)^{2} =  \sum_{j} x^{2}_j + \sum_{j} x'^{2}_j - 2 \sum_{j} x_j x'_j
\leq 2R^2 - 2 \langle \bx, \bx' \rangle  =4R^{2}.
  \end{aligned}
 \end{equation}

\section{Calculations for linear/logistic regression --- sufficient statistics}

\subsection{Sufficient statistics based posterior inference under LDP}
\label{sec:posterior_calculations_GLM}
Assume that we are directly modelling the sum of locally perturbed sufficient statistics $\bZ=\bS+\boldEta$ as a sum of two sums $\bS=\sum_{i} \bs_i, \boldEta= \sum_{i} \bzeta_i$. $\bzeta_i \in \mathcal{N}^{{d+2 \choose d}}(\bzero,\bSigma_{s}), i \in [N]$. Let $\btheta_,\bSigma$ be the parameters of a regression task at hand and $\bmu_{s},\bSigma_{s}$ be the moments of normal approximation of $\bS$. $\bSigma_{*}$ is the diagonal co-variance  matrix of Gaussian noise. 

\begin{equation}
\begin{aligned}
 \Pr[\btheta, \bSigma| \bZ] \propto \Pr[\btheta, \bSigma, \bZ] &=\int_{\Omega(\bS)} \Pr[\btheta,\bSigma,\bS, \bZ] \dd\bS \\ &= \int_{\Omega(\bS)} \Pr[\btheta] \Pr[\bSigma] \Pr[\bS\mid \btheta,\bSigma]  \Pr[\bZ \mid \bS] \dd\bS \\  &= \Pr[\btheta] \Pr[\bSigma] \int_{\Omega(\bS)} \Pr[\bZ \mid \bS]  \Pr[\bS\mid \btheta,\bSigma]   \dd\bS \\ & \approx \Pr[\btheta] \Pr[\bSigma] \int_{\Omega(\bS)}  \mathcal{N} (\bZ; \bS, N\bSigma_{*}) \mathcal{N} (\bS; \bmu_{s}, \bSigma_{s}) \dd \Omega(\bS) \\ &= \Pr[\btheta] \Pr[\bSigma] \int_{\Omega(\bS)}  \Pr[\bZ, \bS|\bSigma_{s},\bmu_{s},N\bSigma_{*}] \dd \Omega(\bS) \\ &=   \Pr[\btheta]\Pr[\bSigma] \Pr[\bZ|N\bSigma_{*},\bmu_{s},\bSigma_{s}] \\ &= \Pr[\btheta] \Pr[\bSigma]\mathcal{N}(\bZ;\bmu_{s},\bSigma_{s}+N\bSigma_{*})
 \label{eq:priv_model_2} 
\end{aligned}
\end{equation}

\subsection{Logistic regression}\label{sec:sufficient_statistics_lr} Given $\bx \in \mathbb{R}^{d}, y \in \{-1,1\}$, \cite{PASS:17} show that the approximate sufficient statistics for logistic regression are 
\begin{align*}
\bt(\bx,y)=[1,vec(\bx \bx^{T})y^{2},\bx y].    
\end{align*}

We define the functions $\bt_1:\mathbb{R}^{ d} \rightarrow \mathbb{R}^{{d+2 \choose 2}}$ and $\bt_2:\mathbb{R}^{ d} \rightarrow \mathbb{R}^{d}$
 as
\begin{equation} \label{eq:t_12}
    \begin{aligned}
             \bt_1(\bx) &= \bx, \\
        \bt_2(\bx) &= \begin{bmatrix} x_1^2,  \ldots , x_d^2, 
          x_1x_2, \ldots ,  x_{d-1} x_d \end{bmatrix}^T.  
    \end{aligned}
\end{equation}
We use the following results from \cite{KJKKH:21}.

\begin{lemma} \label{lem:t_12}
Let $\bt_1$ and $\bt_2$ be defined as in \eqref{eq:t_12} and let $\sigma_1,\sigma_2>0$.
Let $\bs_1 =  y \bt_1(\bx)$ and $\bs_2 =  \bt_2(\bx)$. Consider the mechanism 
$$
\mathcal{M}(\bs) = \begin{bmatrix} \bs_1 \\ \bs_2  \end{bmatrix} + 
\mathcal{N}\left(0, \begin{bmatrix} \sigma_1^2 \bI_d & 0 \\ 0 & \sigma_2^2 \bI_{d_2} \end{bmatrix} \right),
$$
where $d_2 = {d + 2 \choose 2}$. 
Assuming $|| \bx ||_2 \leq R$, the tight $(\epsilon,\delta)$-DP for $\mathcal{M}$ 
is obtained by considering a Gaussian mechanism with noise variance $\sigma_1^2$ and sensitivity
$$
\Delta_{2}(\bt) = \sqrt{  \frac{\sigma_2^2}{2 \sigma_1^2} + 2 R^2 + 2 \frac{\sigma_1^2}{ \sigma_2^2} R^4    }.
$$
\end{lemma}
We use the following consequence of Lemma~\ref{lem:t_12}.

\begin{corollary}
\label{cor:logistic_regression}
In the special case $R=1$ and $\sigma_1 = \sigma_2 = \sigma$, by Lemma~\ref{lem:t_12}, the optimal $(\epsilon,\delta)$ is obtained by considering the Gaussian mechanism
with noise variance $\sigma^2$ and sensitivity $\Delta_{2}(\bt) = \sqrt{ 4 \tfrac{1}{2} }$.
\end{corollary}

\section{Additional Plots}

%\begin{figure}
%    \centering
%    \includegraphics[width=\columnwidth]{figures/AUC_adult_data.pdf}
%    \caption{Logistic regression: Figure shows AUC computed over $450$ test points for several $\epsilon$ and $c$ values for the adult dataset consisting $3000$ points with 5 features. }
%    \label{fig:logistic_regression_input_c}
%\end{figure}

\begin{figure*}[!ht]
  \centering
   \includegraphics[width=1.0\linewidth]{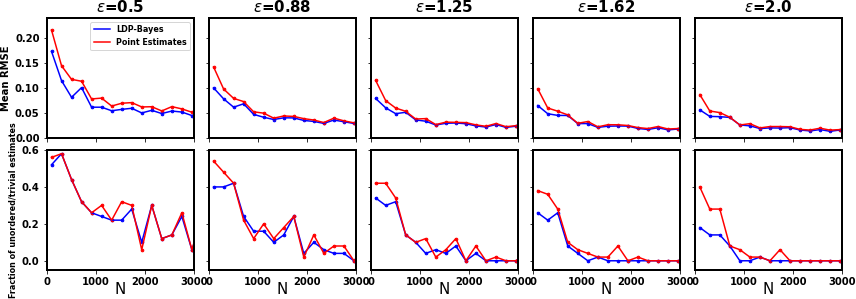}   
   \caption{The top plot compares the mean RMSE across 50 independent runs for for LDP posterior means  and $L_2$ approximations of LDP point estimates.The true histogram parameters ([0.7,0.2,0.1]) as used as the ground truth for computing RMSE.
   The bottom plot shows the fraction of unusable (trivial, unordered) solutions produced by the Bayesian and the point estimate model. }
      \label{fig:point_estimates_vs_posteriors}
\end{figure*}

\begin{figure}[!bt]
    \centering
    \includegraphics[scale=0.5]{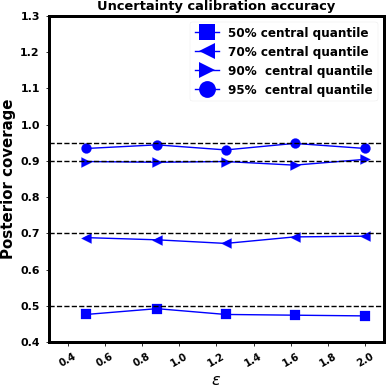}   
    \caption{Histogram discovery: We repeat the inference 100 times for $N=1000$ and $d=6$. For each run, we add-up the number of true $\btheta_i, i \in [d]$'s in 50\%, 70\%, 90\%, and 95\% posterior mass, and divide the final count across executions by $100d$.}
%    \caption{Histogram discovery: We repeat the inference for 100 times for multinomially distributed data perturbed with OUE for various $\epsilon$ values. In each independent execution, we generate $1000$ samples from multinomial distribution with random $\btheta$ of $d=6$  drawn from a Dirichlet distribution as prior. For each run, we add-up the number of true $\btheta_i, i \in [d]$'s in 50\%, 70\%, 90\%, and 95\% posterior mass, and divide the final count across executions by $100d$. Only a small fraction of  posteriors for naive $\theta$ made any of the quantiles.}
         \label{fig:OUE}
\end{figure}

%\begin{figure*}[!ht]
%  \centering
  % \includegraphics[width=1.\linewidth]{adult_LDP_1.png}   
 % \caption{Logistic regression: Comparison of non-private and DP posteriors for the adult dataset with $N= 48K, d= 7$ for various $\epsilon$s. The posterior std. deviation reduce as we increase $\epsilon$. }
   %           \label{fig:lr_adult_data_posteriors}
%\end{figure*}

\begin{figure}[!bt]
    \centering
    \includegraphics[scale=0.5]{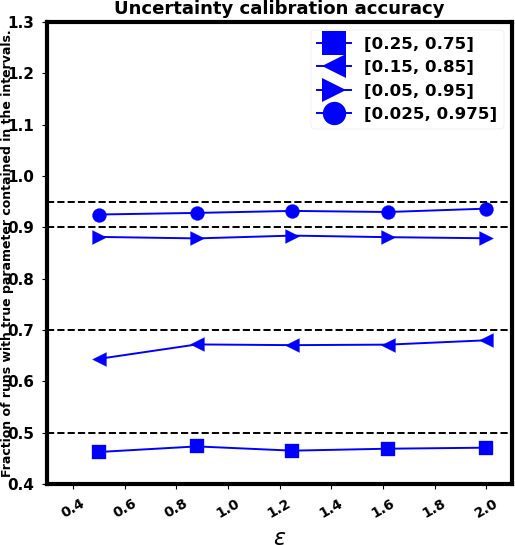}   
    \caption{Exponential parameter discovery: For $N=1000$, we repeat the inference for 80 times for Exponentially distributed data (clipped to $5$) perturbed with Laplace noise for various $\epsilon$ values. In each independent run, we draw $\theta$ from prior. We compute the fraction of runs that included the true $\btheta$ in 50\%, 70\%, 90\%, and 95\% posterior mass. The legend shows the quantile intervals that captures 50\%, 70\%, 90\%, and 95\% posterior mass.}
    \label{fig:exponential}
\end{figure}

\end{document}